\documentclass{imamat}
\gridframe{N}

\jno{dri017}

\usepackage{modamsthm,modnatbib}

\usepackage{latexsym,amssymb,lastpage}
\usepackage{graphicx,amsfonts,gensymb}
\usepackage{times,mathptmx,bm,amsmath}
\usepackage{dcolumn}
\usepackage{dsfont,upgreek,color}
\usepackage{overpic, xcolor}
\usepackage[FIGTOPCAP,nooneline]{subfigure}

\renewenvironment{proof}[1][Proof]{\noindent\textit{#1. } }{\hfill$\square$}

 \newtheoremstyle{theorem}{6pt}{6pt}{\rm}{}{\sffamily}{ }{ }{}
 \theoremstyle{theorem}

  \newtheoremstyle{thm}{6pt}{6pt}{\rm}{}{\sffamily}{ }{ }{}
 \theoremstyle{thm}

 \newtheoremstyle{lemma}{6pt}{6pt}{\rm}{}{\sffamily}{ }{ }{}
 \theoremstyle{lemma}
 \newtheorem{lemma}{\sc Lemma}[section]
 \newtheoremstyle{lem}{6pt}{6pt}{\rm}{}{\sffamily}{ }{ }{}
 \theoremstyle{lem}

\newtheoremstyle{case}{6pt}{6pt}{\rm}{}{}{. }{ }{}
 \theoremstyle{case}

 \newtheoremstyle{statement}{6pt}{6pt}{\rm}{}{\sffamily}{ }{ }{}
\theoremstyle{statement}

 \newtheoremstyle{corollary}{6pt}{6pt}{\rm}{}{\sffamily}{ }{ }{}
 \theoremstyle{corollary}

  \newtheoremstyle{defi}{6pt}{6pt}{\rm}{}{\sffamily}{ }{ }{}
 \theoremstyle{defi}

  \newtheoremstyle{cor}{6pt}{6pt}{\rm}{}{\sffamily}{ }{ }{}
 \theoremstyle{cor}

\newtheoremstyle{example}{6pt}{6pt}{\rm}{}{\sffamily}{ }{ }{}
\theoremstyle{example}

\newtheoremstyle{remark}{6pt}{6pt}{\rm}{}{\sffamily}{ }{ }{}
\theoremstyle{remark}

\newtheoremstyle{approximation}{6pt}{6pt}{\rm}{}{\sffamily}{ }{ }{}
\theoremstyle{approximation}

\newtheoremstyle{scheme}{6pt}{6pt}{\rm}{}{\sffamily}{ }{ }{}
\theoremstyle{scheme}

\newtheoremstyle{Algorithm}{6pt}{6pt}{\rm}{}{\sffamily}{ }{ }{}
\theoremstyle{Algorithm}

 \newtheoremstyle{Remark}{6pt}{6pt}{\rm}{}{\sffamily}{ }{ }{}
 \theoremstyle{Remark}

\newtheoremstyle{Lemma}{6pt}{6pt}{\rm}{}{\sffamily}{ }{ }{}
\theoremstyle{Lemma}

\newtheoremstyle{Assumption}{6pt}{6pt}{\rm}{}{\sffamily}{ }{ }{}
\theoremstyle{Assumption}

\newtheoremstyle{Proposition}{6pt}{6pt}{\rm}{}{\sffamily}{ }{ }{}
\theoremstyle{Proposition}

\newtheoremstyle{prop}{6pt}{6pt}{\rm}{}{\sffamily}{ }{ }{}
\theoremstyle{prop}

\newtheoremstyle{rem}{6pt}{6pt}{\rm}{}{\sffamily}{ }{ }{}
 \theoremstyle{rem}

\newtheoremstyle{hypo}{6pt}{6pt}{\rm}{}{\sffamily}{ }{ }{}
 \theoremstyle{hypo}

  \newtheoremstyle{Step}{6pt}{6pt}{\rm}{}{}{ }{ }{}
 \theoremstyle{Step}

 \newtheoremstyle{lema}{6pt}{6pt}{\rm}{}{\sffamily}{ }{ }{}
 \theoremstyle{lema}

\newcommand{\rhoden}{\text{\usefont{OML}{cmr}{m}{n}\symbol{37}}}
\renewcommand{\theequation}{\thesection.\arabic{equation}}
\numberwithin{equation}{section}
\newcommand{\vect}[1]{\mathbf{#1}}
\newcommand{\vectornorm}[1]{\left|\left|#1\right|\right|}
\newcommand{\eps}{\varepsilon}
\newcommand{\p}{\partial}
\newcommand{\dd}{\mathrm{d}}
\newcommand{\epuck}{\emph{E-Puck} }
\newcommand{\epuckper}{\emph{E-Puck}}
\newcommand{\tar}{\mathcal{T}}
\newcommand{\refl}{\mathcal{R}}
\newcommand{\ntar}{\vect{n}_\tar}
\newcommand{\nref}{\vect{n}_\refl}
\newcommand{\bdtar}{\partial\Omega_\tar}
\newcommand{\bdref}{\partial\Omega_\refl}
\newcommand{\D}{\mathrm{d}}
\newcommand{\DD}[1]{\,\D #1}

\newcommand{\pd}[2]{\dfrac{\partial#1}{\partial#2}}
\newcommand{\nablax}{\nabla_{\vect{x}}}
\newcommand{\Laplace}{\Delta}
\newcommand{\R}{\mathbb{R}}
\newcommand{\unit}[1]{\,\mathrm{#1}}
\newcommand{\changed}[1]{{#1}}

\begin{document}

\title{Mathematical Modelling of Turning Delays in Swarm Robotics}

\author{{\sc Jake P. Taylor-King, Benjamin Franz, Christian A. Yates 
and Radek Erban\raise 1mm \hbox{\thanks{e-mail:erban@maths.ox.ac.uk}}}
\\[6pt]
Mathematical Institute, University of Oxford, \\[1pt]
Andrew Wiles Building, Radcliffe Observatory Quarter\\[1pt]
Woodstock Road, Oxford, OX2 6GG, United Kingdom \\[6pt]
}

\pagestyle{headings}

\markboth{TAYLOR-KING, FRANZ, YATES AND ERBAN}{\rm 
Mathematical Modelling of Turning Delays in Swarm Robotics}

\maketitle

\begin{abstract}
{We investigate the effect of turning delays on the behaviour of groups of 
differential wheeled robots and show that
the group-level behaviour can be described by a transport equation
with a suitably incorporated delay. The results of our mathematical
analysis are supported by numerical simulations and experiments with \epuck robots.
The experimental quantity we compare to our revised model is the mean time for 
robots to find the target area in an unknown environment.
The transport equation with delay better predicts
the mean time to find the target than the standard transport 
equation without delay.}
{Velocity jump process, Swarm robotics, Transport equation with delay}
\end{abstract}        

\section{Introduction}
\label{sec:introduction}
Much theory has been developed for the coordination and control of 
distributed autonomous agents, where collections of robots are acting 
in environments in which only short-range communication is possible \citep{RW}. 
By performing actions based on the presence or absence of signals, 
algorithms have been created to achieve some greater group level task; 
for instance, to reconnoitre an area of interest whilst collecting 
data or maintaining formations \citep{DOK}. In this paper, we will
investigate an implementation of searching algorithms, similar to those
used by flagellated bacteria, in a robotic system.

Many flagellated bacteria such as \emph{Escherichia coli (E. coli)} use 
a run-and-tumble searching strategy in which movement consists of 
more-or-less straight runs interrupted by brief tumbles \citep{Berg:1983:RWB}.  
When their motors rotate counter-clockwise the flagella form a 
bundle that propels the cell forward with a roughly constant speed; 
when one or more flagellar motors rotate clockwise the bundle flies 
apart and the cell `tumbles' \citep{kim2003msm}. Tumbles 
reorient the cell in a more-or-less uniformly-random direction
(with a slight bias in the direction of the previous run) for the next
run \citep{Berg:1972:CEC}. In the absence of signal gradients the random 
walk is unbiased, with a mean run time $\sim 1\, \text{sec}$ and a tumble time 
$\sim 0.1\,\text{sec}$. However, when exposed to an external signal gradient, 
the cell responds by increasing (decreasing) the run length when moving 
towards (away from) a favourable direction, and therefore the random 
walk is biased with a drift in that direction
\citep{Berg:1975:HBS,Koshland:1980:BCM}. Similar behaviour can
be observed in  swarms of animals avoiding 
predators and coordinating themselves within a group \citep{Couzin}. 

The behaviour of \emph{E. coli} is often modelled as a velocity jump
process where the time spent tumbling is neglected as it is much smaller
than the time spent running 
\citep{Othmer:1988:MDB,Erban:2004:ICB}.
In such a velocity jump process, particles follow a given velocity
$\vect{u}$ from a set of allowed velocities $V\subset\mathbb{R}^d$, $d=2,3,$ 
for a finite time. The particle changes velocity probabilistically according 
to a Poisson process with intensity $\lambda$, i.e. the mean run-duration 
is $1/\lambda$. A new velocity 
$\vect{v}$ is chosen according to the turning kernel 
$T(\vect{v},\vect{u}): V\times V \to {\mathds R}$.
Formally the turning kernel represents the probability of 
choosing $\vect{v}$ as the new velocity given that
the old velocity was $\vect{u}$. Therefore, it is necessary 
that $\int_V{T(\vect{v},\vect{u}) \, \mbox{d} \vect{v}}=1$ and $T \geq 0$.

Denoting by $p(t, \vect{x},\vect{v})$ the density
of bacteria which are, at time $t$, at position $\vect{x}$
with velocity $\vect{v}$, the velocity jump process can be described
by the transport equation \citep{Othmer:1988:MDB}
\begin{equation}
\frac{\partial p}{\partial t}(t, \vect{x},\vect{v}) +
\vect{v}\cdot\nabla_{\vect{x}} p(t, \vect{x},\vect{v}) = -\lambda
p(t, \vect{x},\vect{v}) 
+\changed{ \lambda} \int_{V} T(\vect{v}, \vect{u}) \, p(t, \vect{x},\vect{u}) \, \dd
\vect{u}. \label{eq:classicalvjp}
\end{equation}
Assuming that $\lambda$ and $T$ are constant, one can show that 
the long-time behaviour of the density 
$\rhoden(t, \vect{x}) = \int_V  p(t, \vect{x},\vect{v}) \, \mbox{d} \vect{v}$
is given by the diffusion equation \citep{Hillen:2000:DLT}.
If $\lambda$ depends on an external signal (e.g. nutrient concentration), 
then the resulting velocity jump process is biased and its long
time behaviour can be described by a drift-diffusion equation for 
$\rhoden$ \citep{Othmer:2002:DLT,Erban:2005:STS}.

In this paper, we will study an experimental system based on \epuck 
robots \citep{ePuck}. We programme these differential wheeled robots 
to follow a run-and-tumble searching strategy in order to find a given target set. 
\changed{In the first set of experiments, we }concentrate on the simplest possible scenario: 
an unbiased velocity jump process in two spatial dimensions with the fixed 
speed $s\in\mathbb{R}^+$, the constant mean run time 
$\lambda^{-1} \in\mathbb{R}^+$, and the turning kernel which is independent of $\vect{u}$
\begin{equation}
\label{eq:turningKernel}
	T(\vect{v}, \vect{u}) 
	= \frac{\delta(\vectornorm{\vect{v}}-s)}{2\pi s}\,. 
\end{equation}
A special feature of the \epuck robots is that they can perform turns on 
the spot 
as in the classical velocity jump process described by \eqref{eq:classicalvjp}.
In this paper, we will investigate in how far 
\eqref{eq:classicalvjp} presents a good description
of the behaviour of the robotic system and we will develop an extension of
\eqref{eq:classicalvjp} that results in a better match between experimental
data and mathematical model. \changed{We then apply this extended velocity jump theory
to a biased random walk through the incorporation of signals into 
the experimental set up.}

The paper is organized as follows: in Section \ref{sec:experiments}, 
we introduce the experimental system as well as the obtained data.
This data is compared to the classical velocity jump theory. In 
Section \ref{sec:delay}, we extend the velocity jump theory to 
include finite turning times \changed{for unbiased random walks} and compare 
it to our experimental data, showing a much improved match. 
\changed{This new theory is in Section~\ref{sec:signal} applied to a situation
with an external signal and therefore a biased random walk.}
We conclude our paper, in Section \ref{secdiscussion}, by discussing the implications of our results .

\section{Velocity jump processes in experiments with robots}
\label{sec:experiments}
Equation \eqref{eq:classicalvjp} introduced the density behaviour 
of the general velocity jump process that we are aiming to investigate 
using the experimental set-up described
in Section~\ref{subsec:epuck}. 
In particular, we will \changed{initially} concentrate on a simple unbiased velocity jump process 
with the fixed speed $s\in\mathbb{R}^+$, the mean run duration
$\lambda^{-1} \in\mathbb{R}^+$ and the turning kernel (\ref{eq:turningKernel}). 
\changed{In Section~\ref{sec:signal} we will present situations, where the
turning frequency changes according to an external signal, as is indeed common 
in biological applications \citep{Erban:2005:STS}.}
This fixed-speed velocity jump process can be viewed as 
a starting point 
for considering more complex searching algorithms. We will demonstrate 
that by including a small modification (the introduction of a delay 
to the turning kernel), we can alter this simple velocity jump process 
so that it models the behaviour of the \epuck robots.

We are interested in comparing the idealised velocity jump process, 
given in \eqref{eq:classicalvjp}--\eqref{eq:turningKernel}, to 
robotic experiments. Due to a restriction in numbers of
robots, one cannot feasibly talk about a ``density'' of robots that 
could be compared
to $p(t, \vect{x}, \vect{v})$ as given in \eqref{eq:classicalvjp}. 
Therefore, our experiments
concentrate on the escape of robots from a given domain. We may interpret 
this as the target
finding ability of the \epuck robots. Using these experiments, we can 
infer data both on the flux 
at the barrier and the exit times and can compare those to numerical 
results of velocity jump
processes in Sections~\ref{subsec:comparison1} and \ref{subsec:comparison1b}.

\subsection{Experimental set-up and procedure}
\label{subsec:epuck}

To obtain the empirical data, an experimental system consisting of 
16 \epuck robots
was used. \epuck robots are small differential wheeled robots with 
a programmable microchip \citep{ePuck}. 
The diameter of each robot is $\eps = 75\,\text{mm}$ with a height of 
$50\,\text{mm}$ and weight of $200\,\text{g}$. Throughout the 
experiments, the speed was chosen to be $s = 5.8 \times 10^{-2} \,\text{m}/\text{sec}$. 
The robots turn
with an angular velocity $\omega = 4.65/\text{sec}$. 
Full specifications along 
with a picture are given in \ref{app:epuck}.

In the experiments, we use a rectangular arena $\Omega$ with walls 
on three of the 4 edges and 
an opening to the target area $\tar$ along the fourth 
edge\footnote{We have $\Omega\cap\tar = \emptyset$, 
but $\overline{\Omega}\cap\overline{\tar} \neq \emptyset$, 
i.e. $\Omega$ and $\tar$ touch but do not overlap.}. 
A diagram of the arena along with the notation used can be seen in 
Figure~\ref{fig:arenaSketch} and a photo 
is shown in Figure~\ref{CollPhotoArena}(b) in \ref{app:epuck}.
When considering such an arena, one has to distinguish between the size 
of the physical arena and
the effective arena (shown in blue in Figure~\ref{fig:arenaSketch}) that 
the robot centres can
occupy. The effective arena used in the experiments has the dimensions 
$L_x = 1.183\,\text{m}$ and 
$L_y = 1.145\,\text{m} = L_x - \eps/2$. The reflective (wall) boundary 
and the target boundary will be denoted as
$\bdref$ and $\bdtar$, respectively, and can be defined as
\begin{equation}
\bdtar = \overline{\Omega}\cap\overline{\tar}\,,
\qquad \bdref = \partial\Omega \backslash \bdtar\,.
\label{boundarypartition}
\end{equation}
Throughout the remainder of the paper, we will also use $\nref$ 
(resp. $\ntar$) to denote the outwards pointing normal on the 
reflective (resp. target) boundary.

During the experiments, robots were initialised inside a removable 
square pen $\Omega_0$ of effective edge length $L_0 = 0.305\,\text{m}$,
shown in Figure~\ref{fig:arenaSketch} and Figure~\ref{CollPhotoArena}(b) 
in \ref{app:epuck}. A short period of free movement within the 
pen before its removal allowed us to reliably release all robots into 
the full domain $\Omega$ at the same time as well as randomising their 
initial positions within the pen. We recorded the exit time for each of the robots, 
when its geometric centre entered the target area $\tar$. Each repetition 
of the experiment was continued
for $300\,\text{sec}$ or until all $16$ robots had left the arena.
\begin{figure}[t] 
\bigskip
\bigskip
\centering
\begin{overpic}[width=0.75\textwidth]{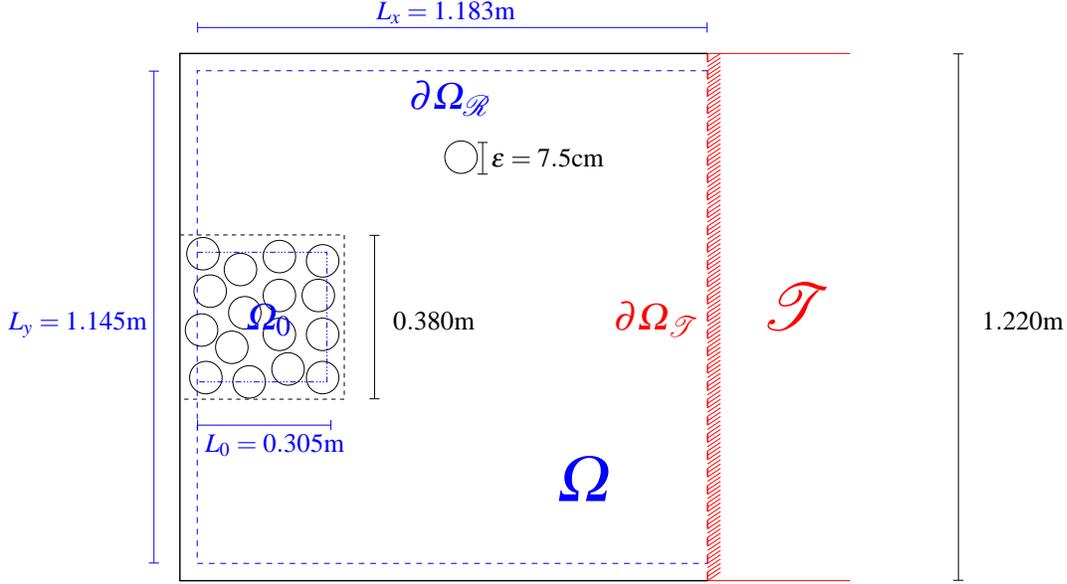}
	\put(75,31){\Huge$\color{red}\tar$}
	\put(50,10){\Huge$\color{blue}\Omega$}
	\put(57,31){\Large$\color{red}\bdtar$}
	\put(32,58){\Large$\color{blue}\bdref$}
	\put(42,51){$\eps = 7.5\text{cm}$}
	\put(-17,31){$\color{blue}L_y = 1.145\text{m}$}
	\put(28,69){$\color{blue}L_x = 1.183\text{m}$}
	\put(7,16){$\color{blue}L_0 = 0.305\text{m}$}
	\put(12,31){\Large$\color{blue}\Omega_0$}
	\put(102,31){$1.220\text{m}$}
	\put(30,31){$0.380\text{m}$}
\end{overpic}
\bigskip
\bigskip
\caption{{\it Schematic showing of the experimental set-up along 
with the notation used throughout this paper. Dotted border lines 
correspond to the effective arena and bold lines to the actual arena. 
For further details see the text. }}
\label{fig:arenaSketch}
\end{figure}%

The robots were programmed using \texttt{C} and a cross-compiling tool, 
with the firmware being transferred onto the
robots via bluetooth. A pseudo-code of the algorithm implemented on the 
robots is shown in Table~\ref{alg:velocityJump}.
This algorithm represents a velocity jump process in the limit as 
$\Delta t\to 0$ \citep{Erban:2006:EFC}, and gives a good
approximation as long as $\lambda\,\Delta t \ll 1$. In the experiments 
we used $\lambda = 0.25\,\text{sec}^{-1}$ implying
a mean run duration of $4\,\text{sec}$ and $\Delta t = 0.1\,\text{sec}$, 
resulting in $\lambda\,\Delta t = 2.5 \times 10^{-2} \ll 1$.
Note that, $s$, $\omega$ and $\lambda$ can be changed on a software 
level on an \epuckper. For $\omega$ we chose
the maximum possible value, whilst for $s$ we chose a value below 
the physical maximum. Choosing a lower velocity means that we mitigate 
the effects of acceleration and deceleration to the running speed since 
the robots cannot do this instantaneously as the basic velocity jump 
model assumes. In a practical setting, one could interpret $s$ and $\omega$ 
as given characteristics of the system, whilst $\lambda$ can be chosen in a 
way that accelerates the
target finding process for the given application with the choice 
of $\lambda$ likely to represent a trade-off
between sampling an area and time spent reorienting.

\begin{table}[ht]
\noindent\framebox{%
\hsize=0.97\hsize
\vbox{
\leftskip 6mm
\parindent -6mm
{\bf [1]} Robot is started at position $\vect{x}(0)\in\Omega_0$. 
Generate $r_1\in[0, 1]$ uniformly at 
random, set $t = 0$ and 
\vskip -2mm
\[
	\vect{v}(0) = s\,\left(\begin{array}{c}\cos 2\pi r_1\\ \sin 2\pi r_1 
	\end{array}\right)\,.
\]

{\bf [2]} Position is updated according to
$\vect{x}(t+\Delta t)=\vect{x}(t)+\Delta t  \ \vect{v}(t)\,.$
\smallskip

{\bf [3]} Generate $r_2\in[0, 1]$ uniformly at random.
If $r_2<\lambda\,\Delta t$, then generate $r_3\in[0, 1]$ uniformly at random
and set 
\vskip -5mm
\[
	\vect{v}(t+\Delta t) = s\,\left(\begin{array}{c}\cos 2\pi r_3\\ 
	\sin 2\pi r_3 \end{array}\right)\,.
\]

{\bf [4]} Set $t=t + \Delta t$ and continue with step [{\bf 2}].
\par \vskip 0.8mm}
}\vskip 1mm
\caption{{\it An algorithmic implementation of the velocity jump process.}}
\label{alg:velocityJump}
\end{table}

In addition to the algorithm in Table~\ref{alg:velocityJump}, robots were also made to implement an 
obstacle-avoidance strategy using the four proximity sensors placed at 
angles $\pm 17.5\degree$ and $\pm 47.5\degree$ from the centre axis
in the front part of the \epuckper. Reflective
turns were carried out based on the signals received at these sensors. 
As the robots are incapable of distinguishing between walls and other 
robots, those
reflections occur whether a robot collides with the wall $\bdref$ or 
another robot. As a consequence we discuss the importance of robot-robot 
collisions on the experimental results \changed{in the next section.}

\subsection{Relevance of collisions for low numbers of robots}
\label{subsec:collisions}

For non-interacting particles which can change 
direction instantaneously, equation \eqref{eq:classicalvjp} 
accurately describes the mesoscopic density through time. 
However, in our experiments the robots undergo reflective 
collisions when they come into close contact, rather than 
passing through or over each other. For a low number 
of particles, we used Monte Carlo simulations to demonstrate 
that collisions are not the dominant behaviour and have 
little effect on the distribution of particles.
\begin{figure}[ht]
\centering
\subfigure[]{
\begin{overpic}[width=0.31\textwidth]{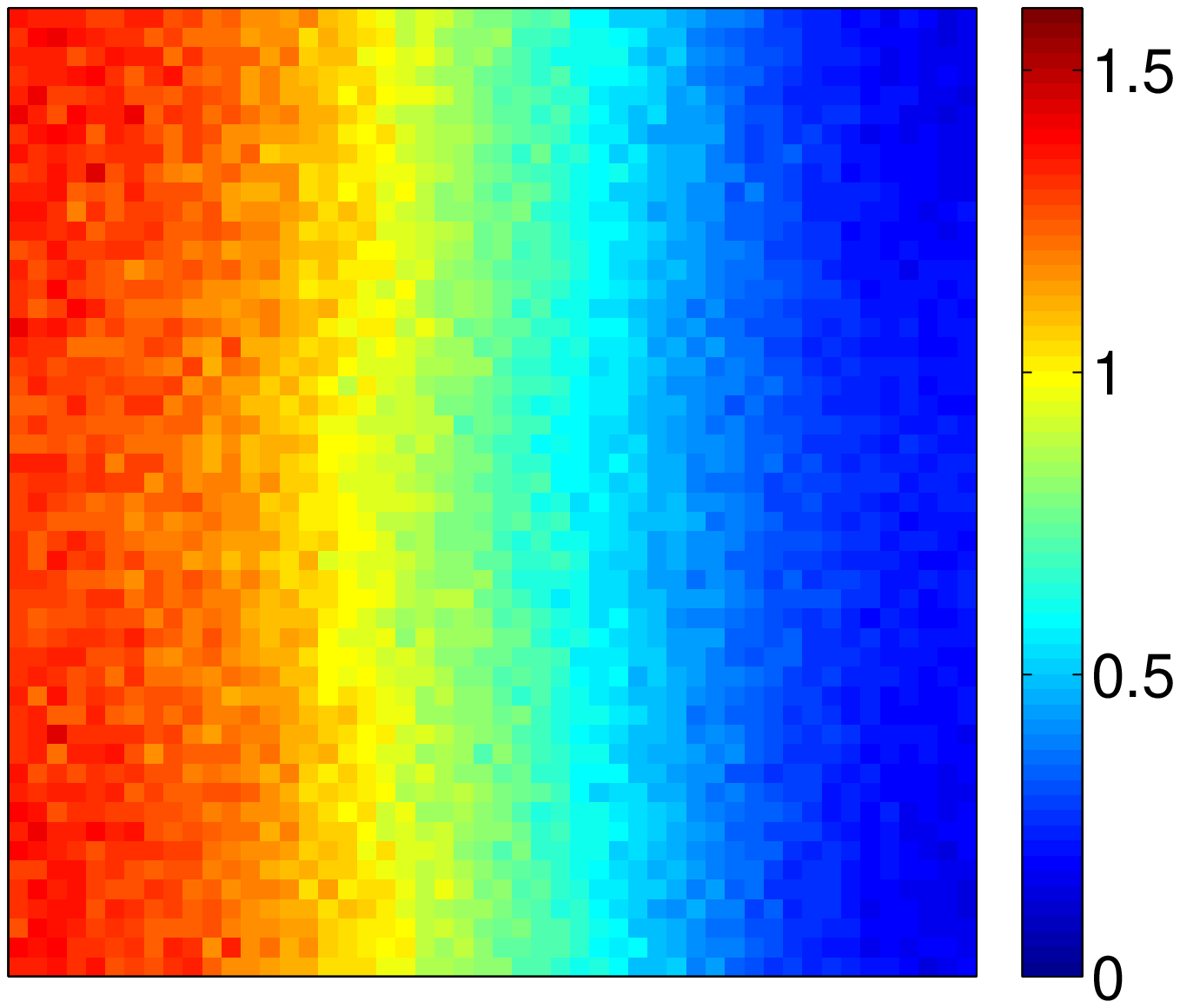}
	\put(5,10){\small$0$}
	\put(-3,41){\small$L_y/2$}
	\put(4,73){\small$L_y$}	
	\put(12,1){\small$0$}
	\put(40,1){\small$L_x/2$}
	\put(74,1){\small$L_x$}
\end{overpic}
\label{subfig:KMC:pointParticles}
}
\subfigure[]{
\begin{overpic}[width=0.31\textwidth]{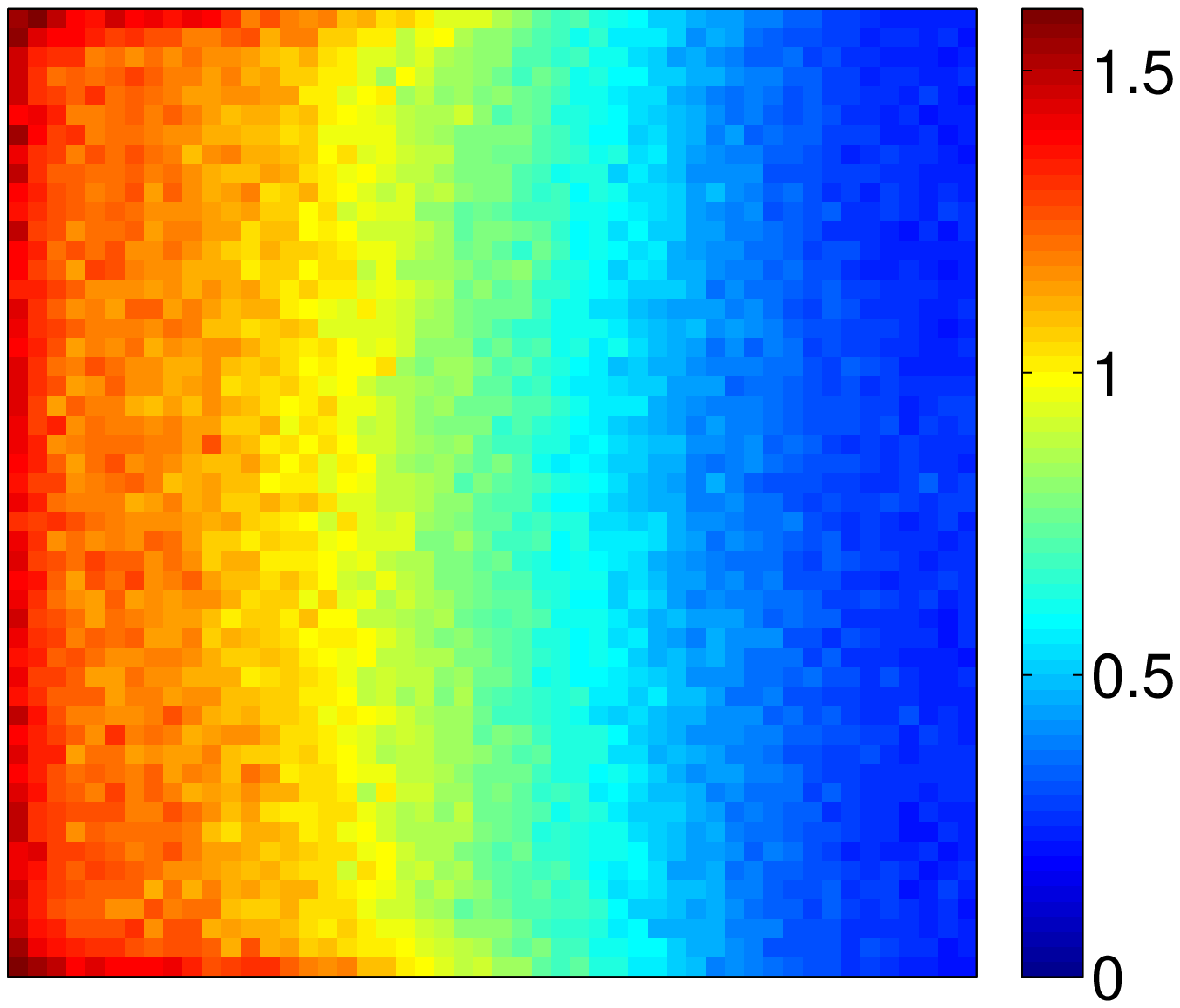}
	\put(5,10){\small$0$}
	\put(-3,41){\small$L_y/2$}
	\put(4,73){\small$L_y$}	
	\put(12,1){\small$0$}
	\put(40,1){\small$L_x/2$}
	\put(74,1){\small$L_x$}
\end{overpic}
\label{subfig:KMC:hardSphere}
}
\subfigure[]{
\begin{overpic}[width=0.31\textwidth]{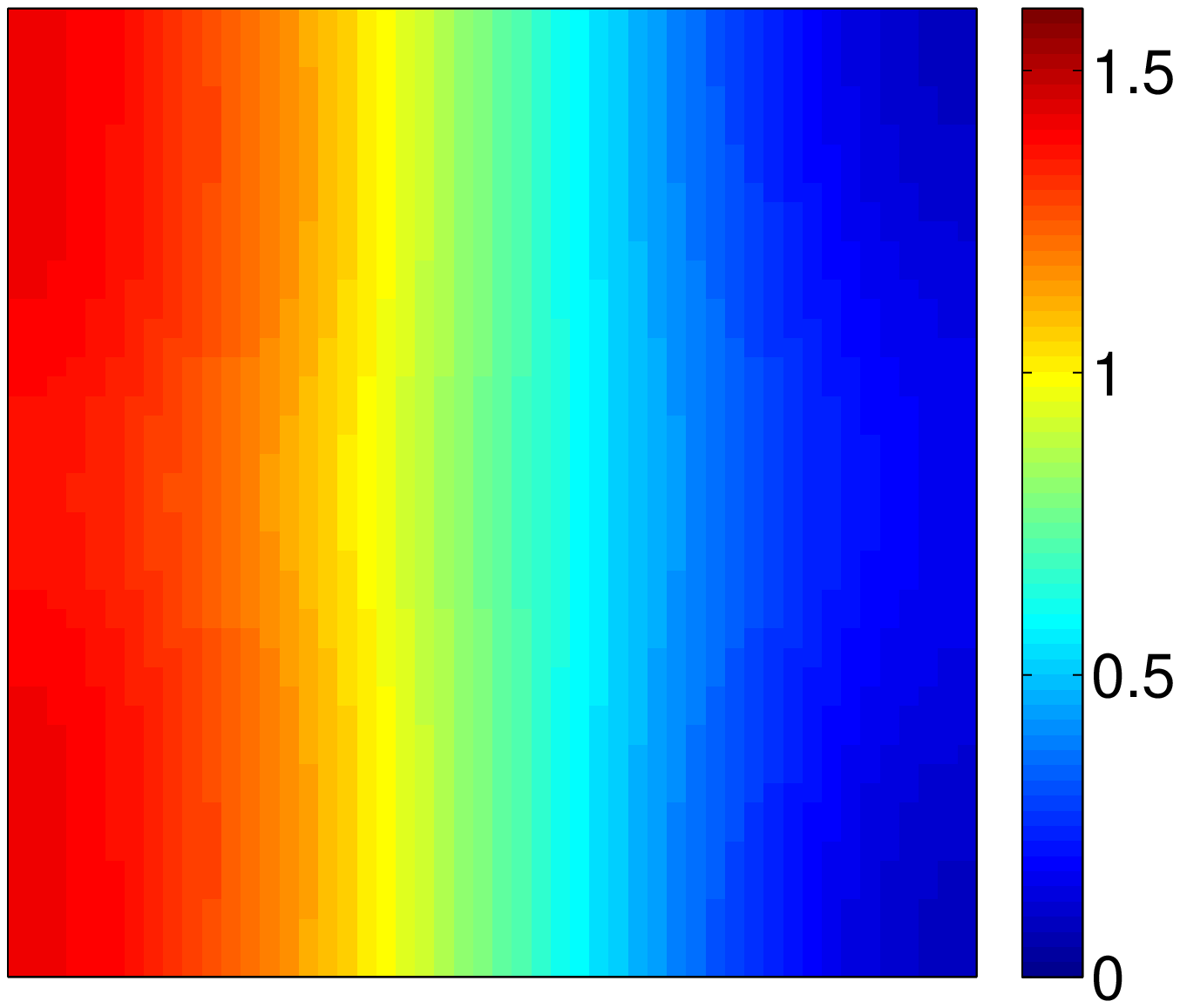}
	\put(5,10){\small$0$}
	\put(-3,41){\small$L_y/2$}
	\put(4,73){\small$L_y$}	
	\put(12,1){\small$0$}
	\put(40,1){\small$L_x/2$}
	\put(74,1){\small$L_x$}
\end{overpic}
\label{subfig:FVM}
}
\caption{{\it Comparison of individual-based simulations with 
$\eqref{eq:classicalvjp}$. Each plot shows the
resulting density at the final time of the simulation, $20\,\text{sec}$.}
\changed{(a) {\it Individual-based simulation using $16\times4\times10^4$ point particles.} 
(b) {\it Individual-based simulation, average over $4\times10^4$ runs using $16$ particles with 
            hard-sphere interactions.} }
(c) {\it Numerical solution to $\eqref{eq:classicalvjp}$ 
using a finite volume method with parameters given in the text.}}
\label{fig:collisions}
\end{figure}
In panels (a) and (b) of Figure~\ref{fig:collisions}, we compare two Monte 
Carlo simulations: (a) in which particles are allowed to pass through 
one another and (b) in which collisions are modelled explicitly. 
In Figure~\ref{fig:collisions}(c) we present 
the solution of equation \eqref{eq:classicalvjp}. This comparison 
demonstrates that the mean density of the underlying process converges
to the solution of transport equation \eqref{eq:classicalvjp}. The parameters 
employed in this model comparison are taken directly from the equivalent
robot experiment; 
$(s,\lambda,\eps) = (5.8 \times 10^{-2} \;\mbox{m}/\mbox{sec}, 
\; 0.250 \; \mbox{sec}^{-1}, 
\; 7.5 \times 10^{-2} \;\mbox{m})$. 
In Figure~\ref{fig:collisions}(c), for the differential equation, we use 
a first-order numerical scheme with $\Delta \theta = \pi/20 $, 
$\Delta x = L_y / 200 $ and $\Delta t = 10^{-2} \,\text{sec}$. 

In the Monte Carlo simulations we initialise particles in the effective 
pen for $20\,\text{sec}$ where they undergo hard-sphere collisions. They 
are then released into the larger arena where in one simulation they are 
point-particles and in the other they undergo reflective collisions as 
hard-spheres. 
Instead of removing particles at the target boundary as shown in 
Figure~\ref{fig:arenaSketch} (as we do in the robot experiments), this 
edge of the domain is closed so that all edges correspond to reflective
 boundary conditions. For transport equation 
 \eqref{eq:classicalvjp}, we model the initial condition as a step function 
 over the pen. These densities are visualised in 
 Figure~\ref{fig:collisions}. Formally, this initial condition 
 can be written as 
 \begin{equation}\label{Eq:comparisonIC}
	p(0, \vect{x}, \vect{v}) = \frac{\chi_{\Omega_0}
	\, \delta(\vectornorm{\vect{v}}-s)}{L_0^2\,2\pi s}\,,
\end{equation}
where $\chi_{\Omega_0}$ denotes the indicator function of the 
initial region $\Omega_0$. The corresponding boundary condition is
$p (t, \vect{x}, \vect{v}) = p(t, \vect{x},\vect{v'})$ for
$\vect{x}\in\bdref$ 
where the reflected velocity $\vect{v'}$ is defined as
\begin{equation} \label{eq:reflection}
\vect{v'} = \vect{v} - 2(\vect{v}\cdot\vect{n}_\Omega)\vect{n}_\Omega\,,
\end{equation}
where $\vect{n}_\Omega$ is the outward pointing normal at the position 
$\vect{x}\in\p\Omega$.

After $20\,\text{sec}$, we record the density in each of the scenarios and 
present the results in Figure \ref{fig:collisions}.
There is minimal visible discrepancy between the Monte Carlo simulations
presented in Figure~\ref{fig:collisions} for our choice of parameter values. 
In order to compare the three simulations given in 
Figure \ref{fig:collisions} we also employed a pairwise Kolmogorov-Smirnov 
test \citep{Peacock1983}. A value (of the Kolmogorov-Smirnov metric) 
close to zero denotes a good fit between the two 
simulations. It corresponds to the probability that one can reject
the hypothesis that the distributions are identical. When 
comparing the two Monte Carlo simulations, a 
value of $2.37 \times 10^{-2}$ was obtained; when comparing equation 
\eqref{eq:classicalvjp} with the hard-sphere Monte Carlo simulation, 
a value of $5.65 \times 10^{-2}$ was obtained; finally when comparing equation 
\eqref{eq:classicalvjp} with the point-particle Monte Carlo simulation, 
a value of $3.40 \times 10^{-2}$ was obtained. 
\changed{This supports the visual observation that all three density 
distributions are all highly similar.}

In the limit where $N\to\infty$, for $N$ being the number of robots, 
transport equation \eqref{eq:classicalvjp} can be altered 
by the addition of a Boltzmann-like collision term 
\citep{SH,Cercignani:1988:BEA}. 
\changed{It can be shown that the effects of collisions between robots 
are negligible for the presented study \citep{Franz_2014}.}

\subsection{Comparison between theory and experiments: loss of mass over time}
\label{subsec:comparison1}

In this and subsequent sections we compare the results of 50 repetitions of 
the experiments described
in Section~\ref{subsec:epuck} with numerical results obtained by 
solving the corresponding mathematical equations.
One way of interpreting the experimental exit-time data is by considering
the expected mass remaining inside the arena $\Omega$ at a given time. 
For the experimental data this quantity is plotted
as a solid (black) line in Figure~\ref{fig:comparison2}\subref{subfig:MassOverTime2}. 
We compare this 
result to the variation of the remaining mass with time from a numerical
solution of \eqref{eq:classicalvjp} combined with the following boundary 
conditions:
\begin{equation}\label{eq:classicalBC}
\begin{alignedat}{2}
p (t, \vect{x}, \vect{v}) & = 0\,, 
&\hspace{0.5cm}& \vect{x}\in\bdtar,\,\vect{v}\cdot\ntar < 0\,, \\ 
p (t, \vect{x}, \vect{v}) & = p(t, \vect{x},\vect{v'})\,,  
&&\vect{x}\in\bdref\,,
\end{alignedat}
\end{equation}
where the reflected velocity $\vect{v'}$ is defined by
(\ref{eq:reflection}). As demonstrated in Section~\ref{subsec:collisions},
such a comparison is reasonable since collisions do not have a
major impact in the parameter regime chosen
here. The initial condition for transport equation \eqref{eq:classicalvjp} 
is identical to the condition given in equation \eqref{Eq:comparisonIC}.
The mass remaining in the domain is then defined as
\begin{equation*}
m(t) = \int_\Omega\int_V p(t, \vect{x}, \vect{v})\,\dd\vect{x}\,\dd\vect{v}\,,
\end{equation*}
and is plotted as a dotted (red) line in 
Figure~\ref{fig:comparison2}\subref{subfig:MassOverTime2}.
The initial mass is normalized to 1.
An obvious observation from Figure~\ref{fig:comparison2}\subref{subfig:MassOverTime2} 
is that the transport equation description does not match the experimental data
well, with the robots exiting the arena significantly slower than 
predicted.
In this figure, we use a first-order finite volume method with 
$\Delta \theta = \pi/20$, 
$\Delta x = 1.183\,\text{m}/200$ and $\Delta t = 10^{-3} \,\text{sec}$ 
in order to solve transport equation \eqref{eq:classicalvjp}.

\begin{figure}
\centering
\subfigure[]{
\includegraphics[width=0.45\textwidth, height=0.35\textwidth]{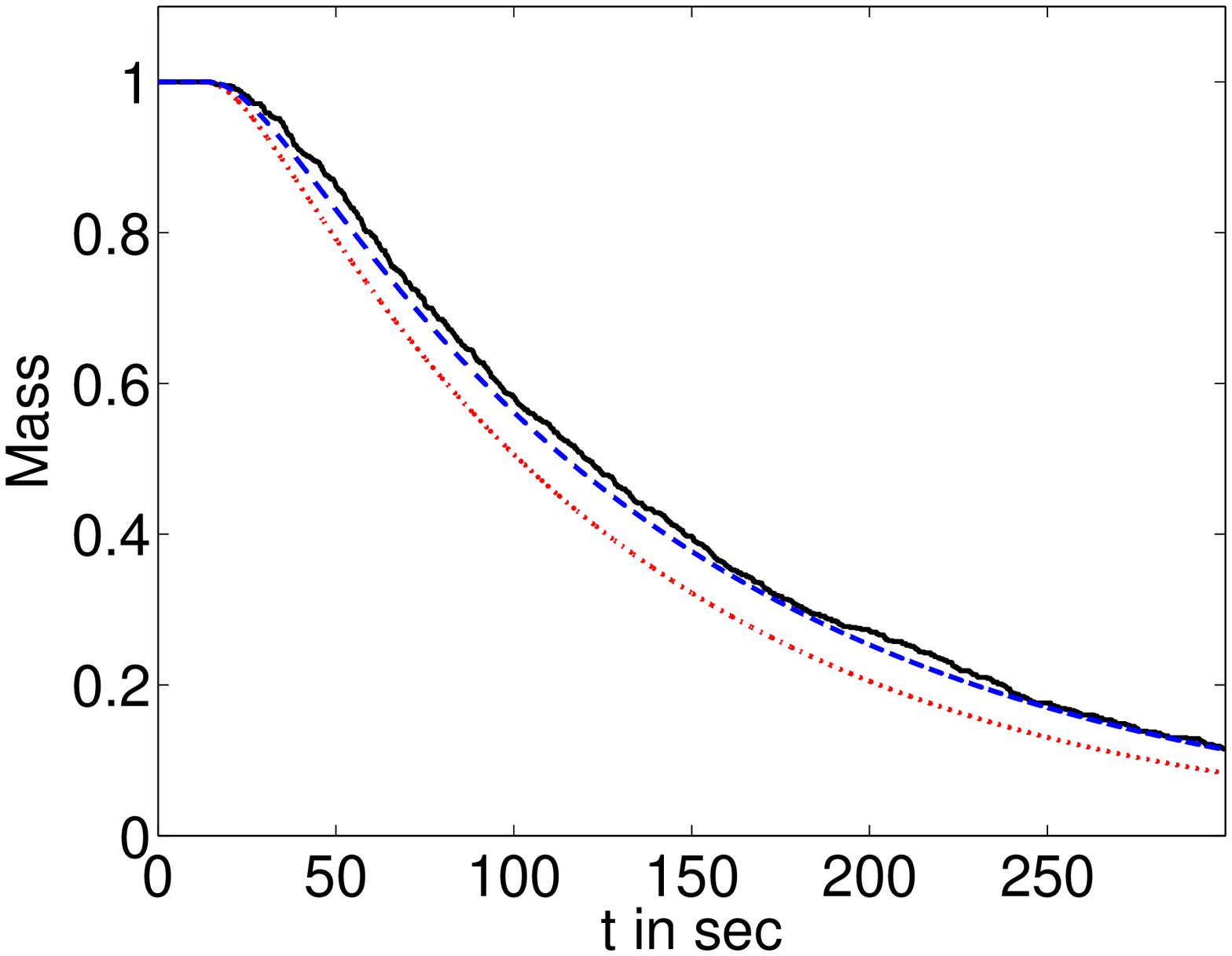}
\label{subfig:MassOverTime2}
}
\subfigure[]{
\includegraphics[width=0.45\textwidth, height=0.35\textwidth]{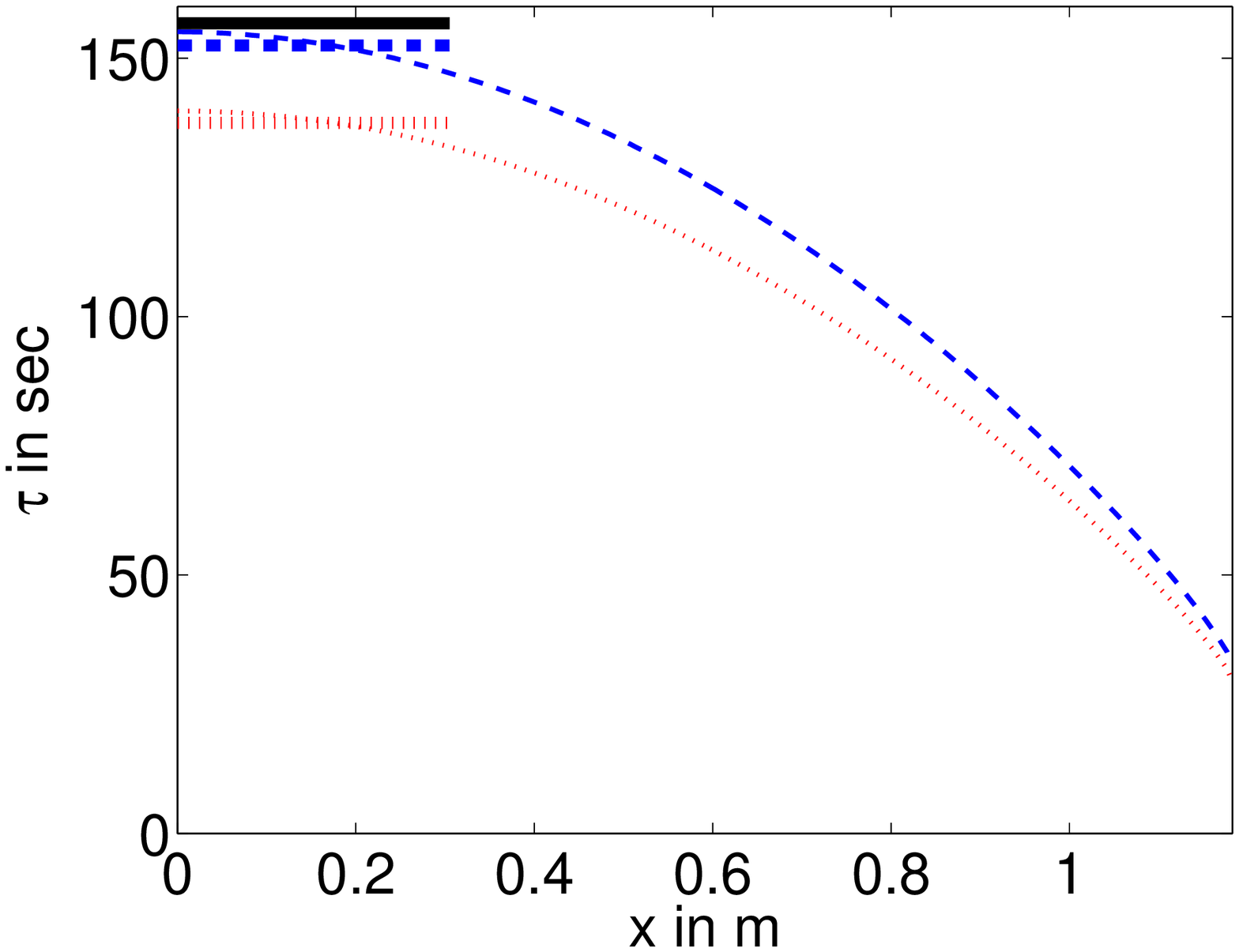}
\label{subfig:MeanExitTime2}
}
\caption{\changed{(Colour available online.) {\it Comparing the experimental 
results (solid black line) to solutions of $\eqref{eq:classicalvjp}$
and $\eqref{eq:restingstate:p}$--$\eqref{eq:restingstate:r}$. 
Panel} (a) {\it shows the relative mass of the system over time.
The dotted line (red) shows the numerical solution to equation $\eqref{eq:classicalvjp}$ 
with boundary conditions $\eqref{eq:classicalBC}$, the dashed line (blue) shows the 
numerical solution of the system of equations 
$\eqref{eq:restingstate:p}$--$\eqref{eq:restingstate:r}$ with boundary 
conditions $\eqref{eq:robotics:restingstate:BC}$.
Panel} (b) {\it shows the mean exit time averaged over all velocity 
directions vs the x-coordinate along the arena edge.
The adsorbing boundary is at x=1.183 m. The 
dotted line (red) shows the mean exit time computed using equation $\eqref{eq:meanExitTime}$ 
with boundary conditions $\eqref{eq:meanExitTimeBC}$,
the dashed line (blue) shows the mean exit time computed using equation 
$\eqref{RevExitTimeProb}$ with boundary conditions 
$\eqref{tauboundaryconditionssim}$. In order to allow direct comparison 
with the experimental data, the shorter bold lines represent the average 
of theoretically derived exit times over the region $\Omega_0$, 
from which the robots 
were released in the experimental scenario. For both plots parameters and numerical 
methods are described in the text.}}}
\label{fig:comparison2}
\end{figure}

\subsection{Comparison between theory and experiments: mean exit time problem}
\label{subsec:comparison1b}

An alternative way to interpret the experimental data is to consider mean 
exit times. \changed{Throughout the experiments only $708$ of the $800$ ($=50\times 16$)
robots left the arena before $t_\mathrm{end} = 300\unit{sec}$. The average
exit time of those $708$ robots was $121.92\unit{sec}$. In order to be able
to compare experimental exit times with the mean exit time problems, it is
necessary to estimate the mean exit time of all $800$ robots. Using
the best exponential fit on the mass over time relation 
(cf. Figure~\ref{fig:comparison2}\subref{subfig:MassOverTime2}),
we can estimate the mean exit time of the remaining $92$ robots to be $424.69\unit{sec}$.
The approximate mean exit time established in the experiments is therefore $156.74\unit{sec}$;
this value is plotted as the solid (black) line in 
Figure~\ref{fig:comparison2}\subref{subfig:MeanExitTime2}.}
In order to be able to compare this value to analytic results, one has 
to reformulate the transport equation 
\eqref{eq:classicalvjp} into a mean exit time problem. Let us therefore define the 
mean exit time $\tau = \tau(\vect{x}_0, \vect{v}_0)$ of a robot that starts at position
$\vect{x}_0\in\Omega$ with velocity $\vect{v}_0\in V$. This mean exit time satisfies the 
following equation
\begin{equation} \label{eq:meanExitTime}
\vect{v}_0\cdot\nabla_{\vect{x}_0} \tau (\vect{x}_0,\vect{v}_0)  
- \lambda \tau ( \vect{x}_0, \vect{v}_0) 
+ \lambda \int_V{T(\vect{u}_0,\vect{v}_0)\tau ( \vect{x}_0, \vect{u}_0)\, 
\dd \vect{u}_0} 
= - 1\,. 
\end{equation}
In Section \ref{sec:delay}, in which delays are modelled, a derivation 
is given for the mean exit time problem; setting the delay term to zero 
allows one to see how equation (\ref{eq:meanExitTime}) is derived.
This so-called ``backwards problem'' satisfies the following boundary conditions
\begin{equation}\label{eq:meanExitTimeBC}
\begin{alignedat}{2}
\tau(\vect{x}_0,\vect{v}_0) & = 0\,, 
&\hspace{0.5cm}&\vect{x}_0\in\bdtar \,, \vect{v}_0 \cdot \ntar > 0\,, \\ 
\tau(\vect{x}_0,\vect{v}_0) & = \tau(\vect{x}_0,\vect{v}_0')\,, 
&&\vect{x}_0\in\bdref\,,
\end{alignedat}
\end{equation}
where $\vect{v}_0'$ is again the reflected velocity with 
respect to $\vect{v}_0$ as defined in \eqref{eq:reflection}. 
Due to the arena shape, by taking the spatial average in the $y$-direction
\begin{equation}\label{classicalmepyav}
\tau_x (x_0,\vect{v}_0) 
= \frac{1}{L_y}\int_{-L_y /2}^{L_y/2} \tau (x_0,y_0,\vect{v}_0) \,\dd y_0, 
\end{equation}
one can further simplify the mean exit time problem. In the case where 
the turning kernel is given by equation \eqref{eq:turningKernel}, 
one can obtain a problem with two parameters $x_0$ and $\theta$, 
where $\theta \in (-\pi,\pi]$ is the angle defining the velocity $\vect{v}_0$
by $\vect{v}_0 = s (\cos(\theta), \sin(\theta)).$
For $\tau_x = \tau_x(x_0,\theta)$
\begin{equation} \label{eq:newFormMET}
\begin{alignedat}{2}
	s \cos(\theta) \frac{\partial \tau_x}{\partial x_0} 
	- \lambda\tau_x + 
	\frac{\lambda}{2\pi}\int_{-\pi}^{\pi} \tau_x(x_0, \phi)\,\dd\phi & = 
	- 1\,, \\
	\tau_x(0, \theta) & = \tau_x(0, \pi-\theta)\,,\\
	\tau_x(L_x, \theta) & = 0\,,&\;\;\theta \in
	\left[-\frac{\pi}{2}, \frac{\pi}{2} \right]. 
\end{alignedat}
\end{equation}
When initial direction cannot be specified, the mean-exit time 
from a given $x$-position is given by 
\begin{equation*}
\frac{1}{2 \pi}\int_{-\pi}^\pi \tau_x(x_0, \theta)\,\dd\theta ,
\end{equation*}
where $\tau_x$ is the solution of \eqref{eq:newFormMET}.
This is plotted as the dotted (red) line in 
Figure~\ref{fig:comparison2}\subref{subfig:MeanExitTime2}. 
The numerical solution was performed using an upwind-scheme
in the $x$-direction with $\Delta x = 1.1825\,\text{m}/200$ and 
an angular discretisation of  
$\Delta \theta = \pi/20$. Additionally, we take the spatial average 
of the mean-exit time from the initial region $\Omega_0$ and 
plot this as the bold dashed line in
Figure~\ref{fig:comparison2}\subref{subfig:MeanExitTime2}. 
This line does not match well with the corresponding average 
mean-exit time found in the robot experiments. \changed{The numerical solution 
of equation \eqref{eq:newFormMET} 
predicted a mean exit time of $137.49\unit{sec}$, meaning an underestimation 
of $19.25\,\text{sec}$ or $12.3\%$ compared to the experimental exit time
of $156.74\unit{sec}$.} In the following section we will extend the classical 
velocity jump theory to improve this match with the experimental data.

\section{Modelling turning delays}
\label{sec:delay}
In Section \ref{subsec:collisions}, we observed that collisions between 
robots does not play a major role in explaining the discrepancy between 
the transport equation \eqref{eq:classicalvjp} and the experimental 
data presented in 
Sections~\ref{subsec:comparison1} and \ref{subsec:comparison1b}. 
As well as assuming independently 
moving particles, the transport equation \eqref{eq:classicalvjp} is 
also predicated on the assumption that the reorientation phase takes 
a negligible amount of time compared to the running phase. Since this 
assumption is not satisfied in our robot  experiments, this section extends the original model through the inclusion of finite turning times.

\subsection{Introduction of a resting state}
\label{subsec:restingState}
Let us initially state two assumptions that apply to the robot 
experiment, but might not extend to velocity jump processes
in biological systems, like the run-and-tumble motion of \emph{E. Coli} 
\citep{Berg:1983:RWB}, which has motivated the searching strategies 
implemented on robots: 

\smallskip
{\leftskip 12mm

\parindent -5mm
\textbf{(a)} a new direction $\vect{v}'\in V$ 
is chosen as soon as the particle enters the reorientation (``tumble") phase; 

\smallskip

\parindent -5mm
\textbf{(b)} the time it takes for a particle
to reorient (``tumble") from velocity $\vect{v}\in V$ to $\vect{v}'\in V$ is 
specified by the function 
$K(\vect{v}', \vect{v}):V\times V\mapsto \mathbb{R}^+$.

\smallskip

\par}

\noindent
Assumption \textbf{(b)} implies that the turning time is 
constant in time and equal for each particle and, in particular, 
does not depend on the particle's history.
For the robots studied in this paper, we can additionally assume 
that reorientation phase is equivalent to a directed
rotation with a constant angular velocity $\omega\in\mathbb{R}^{+}$. 
Therefore, the turning time depends only on
the angle between the current velocity $\vect{v}\in V$ and the new 
velocity $\vect{v}'\in V$ and $K$ takes the form
\begin{equation}\label{eq:DelayKernel}
	K(\vect{v}',\vect{v})
	= \frac{1}{\omega}\arccos 
	\left(\frac{\vect{v}\cdot \vect{v}'}{\vectornorm{\vect{v}}
	\; \vectornorm{\vect{v}'}}\right).
\end{equation}
We now extend the classical model \eqref{eq:classicalvjp} through 
the introduction of a \emph{resting state}
$r(t, \vect{x}, \vect{v}, \eta)$ that formally defines the number 
of particles currently ``tumbling" (turning) towards 
their new chosen velocity $\vect{v}$ and remaining turning time 
$\eta$. The density $p(t, \vect{x}, \vect{v})$ will
now only denote the particles which are at time $t$ in the run phase. The 
update of the extended system is given through
\begin{alignat}{1}\label{eq:restingstate:p}
\frac{\partial p}{\partial t}(t, \vect{x}, \vect{v}) 
+ \vect{v}\cdot \nabla_{\vect{x}} p(t, \vect{x}, \vect{v})
& = 
-\lambda p(t, \vect{x}, \vect{v}) 
+ r(t, \vect{x}, \vect{v}, 0^+)\,,\\
\label{eq:restingstate:r}
\frac{\partial r}{\partial t}(t, \vect{x}, \vect{v}, \eta) 
- \frac{\partial r}{\partial \eta}(t, \vect{x}, \vect{v}, \eta)  
& = 
\lambda \int_V p(t, \vect{x}, \vect{u})\, T(\vect{v}, \vect{u})\, 
\delta(\eta - K(\vect{v}, \vect{u}))\, \D\vect{u}\,.
\end{alignat}
In \eqref{eq:restingstate:p} we can see that running particles 
will enter a tumble phase with rate $\lambda$ and
particles that have finished the tumble signified through 
$\eta = 0$ will re-enter the run-phase. 
Equation \eqref{eq:restingstate:r} represents the linear 
relation between $t$ and $\eta$ and 
shows that particles enter the tumble phase depending on their
newly chosen velocity direction. In order to guarantee 
conservation of mass throughout the system, we introduce the
non-negativity condition for $\eta$ through
\begin{equation*}
	r(t, \vect{x}, \vect{v}, \eta) = 0, 
	\qquad
	\mbox{for}
	\qquad 	t > 0, \quad \vect{x} \in \Omega, \quad 
	\vect{v}\in V \quad \mbox{and} \quad \eta < 0.
\end{equation*}
Additionally, the boundary conditions for the 
system \eqref{eq:restingstate:p}--\eqref{eq:restingstate:r}
are given through
\begin{equation}\label{eq:robotics:restingstate:BC}
\begin{alignedat}{2}
p(t, \vect{x}, \vect{v}) & 
= 0\,, &\hspace{0.5cm}&\vect{x}\in\bdtar\,,\;\vect{v}\cdot\ntar < 0\,,
\;\\	
p(t, \vect{x}, \vect{v}) & = 
-r(t, \vect{x}, \vect{v}, 0^+)/(\vect{v}\cdot\nref)\,, &&
\vect{x}\in\bdref\,,\;\vect{v}\cdot\nref < 0\,,\\
\frac{\partial r}{\partial t}(t, \vect{x}, \vect{v}, \eta) 
- \frac{\partial r}{\partial \eta}(t, \vect{x}, \vect{v}, \eta)
& =  \delta(\eta - K(\vect{v}, \vect{v}'))\,(\vect{v'}\cdot\nref)\,	
p(t, \vect{x}, \vect{v}')\,,&& \vect{x}\in\bdref\,,\;\vect{v}\cdot\nref 
< 0\,,\\
r(t, \vect{x}, \vect{v}, \eta) & = 0\,, 
&& \vect{x}\in\bdref\,,\;\vect{v}\cdot\nref > 0\,,
\end{alignedat}
\end{equation}
where $\vect{v'}$ is the reflected velocity of $\vect{v}$ given by
(\ref{eq:reflection}). In order to show that the system 
\eqref{eq:restingstate:p}--\eqref{eq:restingstate:r} is actually 
consistent, we prove that mass in the system is conserved if 
no target is present.
\begin{lemma}
{\it The total mass in system 
$\eqref{eq:restingstate:p}$--$\eqref{eq:restingstate:r}$ with the boundary
conditions given in $\eqref{eq:robotics:restingstate:BC}$ in the case 
of reflective boundaries everywhere 
($\bdref = \partial\Omega\,,\bdtar = \emptyset$) given through
\begin{equation*}
M(t) = \int_\Omega\int_V p(t, \vect{x}, \vect{v})\,\D\vect{v}\,\D\vect{x}
+ \int_{\overline{\Omega}}\int_V\int_0^\infty r(t, \vect{x}, \vect{v}, \eta)\,
\D\eta\,\D\vect{v}\,\D\vect{x}\,,
	\end{equation*}
	is conserved.
	}
\end{lemma}
\begin{proof}
We define for every point $\vect{x}\in\bdref$ the two subsets 
$V^+$ and $V^-$ of $V$ as follows
\begin{equation}
V^+(\vect{x}) = \left\{ \vect{v}\in V\ : \ \vect{v}\cdot\nref > 0
\right\} \,,\qquad
V^-(\vect{x}) = \left\{ \vect{v}\in V\ : \ \vect{v}\cdot\nref < 0
\right\}\,.
\label{velocitypartition}
\end{equation}
Additionally, let us define
\begin{equation*}
R(t, \vect{x}, \vect{v}) 
= \int_0^\infty r(t, \vect{x}, \vect{v}, \eta)\,\D\eta\,.
\end{equation*}
Integrating \eqref{eq:restingstate:r} with respect to $\eta\in [0, \infty)$, 
we obtain after reordering for $\vect{x}\notin\partial\Omega$:
\begin{equation*}
\frac{\partial R}{\partial t}(t, \vect{x}, \vect{v}) 
= - r(t, \vect{x}, \vect{v}, 0^+) 
+ \lambda\int_Vp(t, \vect{x}, \vect{u})T(\vect{v}, \vect{u})\,\D\vect{u}\,.
\end{equation*}
Hence, for every point $\vect{x}\notin\partial\Omega$ we obtain
\begin{equation*}
\frac{\partial}{\partial t}\left[p(t, \vect{x}, \vect{v}) 
+ R(t, \vect{x}, \vect{v})\right]
= -\lambda p(t, \vect{x}, \vect{v}) 
+ \lambda \int_V p(t, \vect{x}, \vect{u}) T(\vect{v}, \vect{u})\,\D\vect{u}
- \vect{v}\cdot \nabla_{\vect{x}} p(t, \vect{x}, \vect{v})\,.
\end{equation*}
Integrating this with respect to $\vect{x}\in\Omega$ and $\vect{v}\in V$
gives
\begin{equation}\label{eq:res1}
\int_\Omega \int_V \frac{\partial}{\partial t}\left[p(t, \vect{x}, \vect{v})
+ R(t, \vect{x}, \vect{v})\right]\,\D\vect{v}\,\D\vect{x}
= - \int_\Omega\int_V \vect{v}\cdot \nabla_{\vect{x}} 
p(t, \vect{x}, \vect{v})\,\D\vect{v}\,\D\vect{x} \,.
\end{equation}
Using the divergence theorem, we can evaluate the integral on the right
hand side to be
\begin{eqnarray*}
\int_\Omega\int_V \vect{v}\cdot\nabla_x p\,\D\vect{v}\,\D\vect{x} 
& = &
\int_{\partial\Omega}\int_V (\vect{v}\cdot\nref(\vect{x}))\, 
p(t, \vect{x}, \vect{v})\,\D\vect{v}\,\D\vect{x}\\
& = & \int_{\partial\Omega}\int_{V^+(\vect{x})} (\vect{v}\cdot\nref)\, 
p(t, \vect{x}, \vect{v})\,\D\vect{v}\,\D\vect{x}
+ \int_{\partial\Omega}\int_{V^-(\vect{x})} (\vect{v}\cdot\nref)\, 
p(t, \vect{x}, \vect{v})\,\D\vect{v}\,\D\vect{x}\\
& = & \int_{\partial\Omega}\int_{V^+(\vect{x})} (\vect{v}\cdot\nref)\, 
p(t, \vect{x}, \vect{v})\,\D\vect{v}\,\D\vect{x}
- \int_{\partial\Omega}\int_{V^-(\vect{x})} 
r(t, \vect{x}, \vect{v}, 0^+) \,\D\vect{v}\,\D\vect{x}\,,
\end{eqnarray*}
where we have used the second boundary condition 
in \eqref{eq:robotics:restingstate:BC} in the last step. 
Additionally, for $\vect{x}\in\partial\Omega$ and 
$\vect{v}\in V^-(\vect{x})$, we obtain by integrating 
the third boundary condition in \eqref{eq:robotics:restingstate:BC} 
with respect to $\eta\in[0, \infty)$
\begin{equation*}
\frac{\partial R}{\partial t}(t, \vect{x}, \vect{v}) 
= - r(t, \vect{x}, \vect{v}, 0^+) + (\vect{v'}\cdot\nref)
\,p(t, \vect{x}, \vect{v'})\, \changed{.}
\end{equation*}
Integrating this with respect to $\vect{x}\in\partial\Omega$ and 
$\vect{v}\in V$ and using the last boundary condition
in \eqref{eq:robotics:restingstate:BC} we obtain
\begin{eqnarray}
\int_{\partial\Omega} \int_V \frac{\partial R}{\partial t} 
\,\D\vect{v}\,\D\vect{x}
& = & \int_{\partial\Omega} \int_{V^-(\vect{x})} 
\frac{\partial R}{\partial t} \,\D\vect{v}\,\D\vect{x}\nonumber\\
& = & -\int_{\partial\Omega}\int_{V^-(\vect{x})} 
r(t, \vect{x}, \vect{v}, 0^+)\,\D\vect{v}\,\D\vect{x}
+ \int_{\partial\Omega}\int_{V^-(\vect{x})} (\vect{v'}\cdot\nref)
\,p(t, \vect{x}, \vect{v'})\,\D\vect{v}\,\D\vect{x}\nonumber\\
& = & -\int_{\partial\Omega}\int_{V^-(\vect{x})} 
r(t, \vect{x}, \vect{v}, 0^+)\,\D\vect{v}\,\D\vect{x}
+ \int_{\partial\Omega}\int_{V^+(\vect{x})} (\vect{v}\cdot\nref)
\,p(t, \vect{x}, \vect{v})\,\D\vect{v}\,\D\vect{x}. \label{eq:res2}
\end{eqnarray}
Summing up the results from \eqref{eq:res1} and \eqref{eq:res2}, we obtain
$\mbox{d}M/\mbox{d}t = 0$ and hence the total mass $M(t)$ in the system 
is conserved.
\end{proof}

\subsection{Transport equation with turning delays}
\label{subsec:turning:transportEquation}
We eliminate the resting state from system 
\eqref{eq:restingstate:p}--\eqref{eq:restingstate:r}
and derive the generalization of the transport equation 
\eqref{eq:classicalvjp} to a transport equation with 
a suitably incorporated delay. This can be done by solving 
\eqref{eq:restingstate:r} for $r$ using the method 
of characteristics, which results in
\begin{equation}
r(t, \vect{x}, \vect{v}, 0) 
=
r(0, \vect{x}, \vect{v}, t) 
+ 
\lambda \int_V 
T(\vect{v}, \vect{u})
\, p(t - K(\vect{v}, \vect{u}),\vect{x}, \vect{u})
\, H(t - K(\vect{v}, \vect{u}))
\dd\vect{u}\,,
\label{expressionT}
\end{equation}
where $H$ is the Heaviside step function. Let us assume
that $K(\vect{v}, \vect{u})$ is given by (\ref{eq:DelayKernel}).
Then $K(\vect{v}, \vect{u}) \le \pi/\omega$. Considering times
$t > \pi/\omega$, we have $r(0, \vect{x}, \vect{v}, t) = 0$. 
We can now substitute (\ref{expressionT})
into \eqref{eq:restingstate:p} to obtain
\begin{equation}
\label{eq:delayDE}
\frac{\p p}{\p t} \, (t, \vect{x},\vect{v}) + 
\vect{v}\cdot\nabla_{\vect{x}} \, p(t, \vect{x},\vect{v}) = 
-\lambda p(t, \vect{x}, \vect{v}) + 
\lambda \int_V
T(\vect{v},\vect{u})
\,
p(t - K(\vect{v},\vect{u}), \vect{x}, \vect{u})
\, \dd \vect{u},
\end{equation}
for $t > \pi/\omega$. 
Note that \eqref{eq:delayDE} only considers particles in the running 
phase and hence does not strictly conserve mass. 
The boundary conditions for transport equation \eqref{eq:delayDE} are
\begin{equation}\label{eq:delayBC}
\begin{alignedat}{2}
p (t, \vect{x}, \vect{v}) & = 0\,, 
&\hspace{0.5cm} & \vect{x}\in\bdtar\,,\vect{v}\cdot\ntar < 0\,, \\ 
p (t, \vect{x}, \vect{v}) & = p(t - K(\vect{v},\vect{v'}), \vect{x},\vect{v'})\,,  
&& \vect{x}\in\bdref\,,\vect{v} \cdot \nref < 0\,.
\end{alignedat}
\end{equation}

\subsection{Equation for mean-exit time}
\label{subsec:turning:mfpt}

Equation (\ref{eq:delayDE}) can be rewritten as $\mathcal{M}p = 0,$
where the operator $\mathcal{M}$ is given by
\begin{equation}
\mathcal{M}p =
- \frac{\p p}{\p t} 
- \vect{v}\cdot\nabla_{\vect{x}} p
- \lambda p 
+ \lambda \int_V
T(\vect{v},\vect{u})
p(t - K(\vect{v},\vect{u}), \vect{x}, \vect{u})\dd \vect{u}.
\label{operatorM}
\end{equation}
For a forward problem specified by $\mathcal{M}p = 0,$
coupled with initial and boundary conditions, the backward problem 
is given by the adjoint operator $\mathcal{M}^*q = 0,$
with final condition and adjoint boundary conditions \citep{JL}. 
The adjoint operator is given by:
\begin{equation*}
\langle \mathcal{M}p,q \rangle
= \langle p,\mathcal{M}^*q \rangle 
\quad\text{where}
\quad \langle p,q \rangle 
= \int_{-\infty}^{\infty}\int_{\Omega}
\int_V p(t, \vect{x}, \vect{v}) q(t, \vect{x}, \vect{v}) \,
\dd \vect{v} \, \dd \vect{x} \, \dd t.
\end{equation*}
Using integration by parts and the divergence theorem,
we see
\begin{eqnarray}
\langle \mathcal{M}p,q \rangle
&=& 
\int_{-\infty}^{\infty} \int_\Omega \int_V 
\left(
- \frac{\p p}{\p t} 
- \vect{v}\cdot\nabla_{\vect{x}} p
- \lambda p 
+ \lambda \int_V
T(\vect{v},\vect{u})
p(t - K(\vect{v},\vect{u}), \vect{x}, \vect{u})\dd \vect{u}
\right)
q \; \dd \vect{v} \, \dd \vect{x} \, \dd t \nonumber \\ 
&=& 
\int_{-\infty}^{\infty} \int_\Omega \int_V 
p \left(
\frac{\p q}{\p t} 
+ \vect{v}\cdot\nabla_{\vect{x}} q
- \lambda q
+ \lambda \int_V
T(\vect{u},\vect{v})
q(t + K(\vect{u},\vect{v}), \vect{x}, \vect{u}) \dd \vect{u}
\right)
\, \dd \vect{v} \, \dd \vect{x} \, \dd t \nonumber \\
&&+
\int_{-\infty}^{\infty} 
\int_V
\int_{\p\Omega} 
p(t,\vect{x},\vect{v}) q(t,\vect{x},\vect{v}) 
[\vect{v} \cdot \vect{n} ]\, \dd S_x \, \dd\vect{v} \, \dd t
\label{pomeqMstar}
\end{eqnarray}
where we used the boundary conditions
\begin{equation}
\lim\limits_{t \to\infty} p(t,\vect{x},\vect{v}) 
=
\lim\limits_{t \to -\infty} q(t,\vect{x},\vect{v}) \quad 
=
0.
\end{equation}
We will also assume the following boundary conditions
\begin{equation}
\label{backwardbcs}
\begin{alignedat}{2}
 q (t, \vect{x},\vect{v}) & = 
 0\,,  &\hspace{0.5cm}&\vect{x}\in \bdtar\,,\vect{v} \cdot \vect{n}_{\mathcal{T}} > 0 \,, \\ 
q(t , \vect{x},\vect{v}) & = q (t + K(\vect{v}',\vect{v}), \vect{x},\vect{v}') \,, 
&& \vect{x}\in \bdref\,, \vect{v} \cdot \vect{n}_{\mathcal{R}} > 0\,.
\end{alignedat}
\end{equation}
Then the last term in (\ref{pomeqMstar}) is equal to zero as it is
shown in \ref{BCDERIVATION}. Using 
(\ref{pomeqMstar})--(\ref{backwardbcs}) and the variable set 
$(t_0, \vect{x}_0,\vect{v}_0)$ to indicate starting times 
and positions, we can write the backwards equation $\mathcal{M}^*q = 0$
in the following form:
\begin{equation}
- \frac{\p q}{\p t_0}
(t_0, \vect{x}_0,\vect{v}_0) 
-
\vect{v}_0\cdot\nabla_{\vect{x}_0} 
q (t_0, \vect{x}_0,\vect{v}_0)
=
-
\lambda q(t_0, \vect{x}_0, \vect{v}_0) + 
\lambda \int_V{T(\vect{u}_0,\vect{v}_0)q(t_0 + 
K(\vect{u}_0,\vect{v}_0), \vect{x}_0, \vect{u}_0)\dd \vect{u}_0}.
\label{pomMstareq2}
\end{equation}
More precisely, we should write
$q (t_0, \vect{x}_0,\vect{v}_0)
\equiv p (t, \vect{x},\vect{v} \, | \, t_0, \vect{x}_0,\vect{v}_0)$,
i.e. $q$ gives the probability that the particle is at the
position $\vect{x}$ with velocity $\vect{v}$ at time $t$ given
that its initial position and velocity at time $t_0$ were
$\vect{x}_0$ and $\vect{v}_0$, respectively.
Let $\rho \equiv \rho (t,\vect{x}_0, \vect{v}_0)$ 
be the probability that the particle is in $\Omega$ at time 
$t$ given that the initial position and velocity is given as 
$\vect{x}_0$ and $\vect{v}_0$, respectively. Then
$$
\rho (t,\vect{x}_0, \vect{v}_0)
=
\int_\Omega \int_V
p (t, \vect{x},\vect{v} \, | \, 0, \vect{x}_0,\vect{v}_0)
\,\D\vect{v}\,\D\vect{x} =
\int_\Omega \int_V
p (0, \vect{x},\vect{v} \, | \, -t, \vect{x}_0,\vect{v}_0)
\,\D\vect{v}\,\D\vect{x}. 
$$
Substituting $t_0 = -t$ into (\ref{pomMstareq2}) and using
the Taylor expansion, we obtain 
\begin{eqnarray}
\frac{\p \rho}{\p t} (t, \vect{x}_0,\vect{v}_0) 
- \vect{v}_0\cdot\nabla_{\vect{x}_0} \rho (t, \vect{x}_0,\vect{v}_0) 
&=& -\lambda \rho (t, \vect{x}_0, \vect{v}_0) 
+ \lambda \int_V
T(\vect{u}_0,\vect{v}_0)\rho (t - K(\vect{u}_0,\vect{v}_0), 
\vect{x}_0, \vect{u}_0) \, \dd \vect{u}_0 \nonumber \\
&=& -\lambda \rho (t, \vect{x}_0, \vect{v}_0) 
+ \lambda \int_V T(\vect{u}_0,\vect{v}_0)
\, \rho (t, \vect{x}_0, \vect{u}_0)
\, \dd \vect{u}_0
\nonumber \\ 
&&- 
\lambda \int_V
T(\vect{u}_0,\vect{v}_0)
K(\vect{u}_0,\vect{v}_0)
\frac{\p \rho}{\p t} (t, \vect{x}_0, \vect{u}_0)
\, \dd \vect{u}_0 + \dots . 
\label{pomtaylexpM}
\end{eqnarray}
The probability of a single particle leaving $\Omega$ in time interval 
$[t,t+\dd t)$ is $\rho(t,\vect{x}_0,\vect{v}_0) 
- \rho(t + dt,\vect{x}_0,\vect{v}_0) 
\approx -\p \rho/\p t (t,\vect{x}_0,\vect{v}_0)\, \dd t$.
Consequently, the expected exit time is given by 
\begin{equation*}
\tau (\vect{x}_0, \vect{v}_0)
= -\int_{0}^\infty
t\frac{\p \rho}{\p t} (t,\vect{x}_0,\vect{v}_0) \, \dd t 
= \int_{0}^\infty \rho(t,\vect{x}_0,\vect{v}_0) \, \dd t ,
\end{equation*} 
where we use the fact that $\rho(t, \vect{x}_0, \vect{v}_0)\to 0$ as $t\to\infty$.
Integrating (\ref{pomtaylexpM}) over time, we obtain
\begin{equation} \label{dimensionalTauProblem}
\begin{alignedat}{1}
&\vect{v}_0\cdot\nabla_{\vect{x}_0} \tau (\vect{x}_0,\vect{v}_0)  
- \lambda \tau ( \vect{x}_0, \vect{v}_0) + \lambda \int_V{T(\vect{u}_0,\vect{v}_0)\tau ( \vect{x}_0, \vect{u}_0) 
\dd \vect{u}_0} \\ &\hspace{3cm}= 
-
\left( 1 +  \lambda \int_V T(\vect{u}_0,\vect{v}_0) K(\vect{u}_0,\vect{v}_0) 
\dd \vect{u}_0\right)\,,
\end{alignedat}
\end{equation}
where we neglected the higher order terms. By Taylor-expanding the 
boundary terms from equation \eqref{eq:delayBC} and integrating in 
time, we obtain the following boundary conditions
\begin{equation}\label{tauboundaryconditions}
\begin{alignedat}{2}
\tau(\vect{x}_0,\vect{v}_0) &= 0\,,  &\hspace{0.5cm}& 
\vect{x}_0\in\bdtar\,, \;\vect{v}_0 \cdot \ntar > 0\,, \\ 
\tau(\vect{x}_0,\vect{v}_0') &= \tau(\vect{x}_0,\vect{v}_0) + 
K(\vect{v}_0',\vect{v}_0)\,, && 
\vect{x}_0\in\bdref\,, \;\vect{v}_0 \cdot \nref > 0\,,
\end{alignedat}
\end{equation}
where the reflected velocity $\vect{v}_0'$ is given by (\ref{eq:reflection}),
i.e. $\vect{v}_0' = \vect{v}_0 - 2(\vect{v}_0\cdot \nref)\nref$. 

\subsection{Comparison between the transport
equation theory with delays and experimental results}
\label{subsec:comparison2}

Let us now compare the extended theory developed in 
Sections~\ref{subsec:restingState}--\ref{subsec:turning:mfpt}
to the experimental data using the same approach as 
in Sections~\ref{subsec:comparison1} and \ref{subsec:comparison1b}.
For the case of the arena given in Figure~\ref{fig:arenaSketch}, 
we write $\Omega = (0,L_x) \times (-L_y/2,L_y/2)$ 
and $\mathcal{T} = (L_x,\infty)\times (-L_y/2,L_y/2)$ 
and we simplify equation (\ref{dimensionalTauProblem}) by 
integrating over the $y$-direction to obtain an average value 
for $\tau$ for our position along the $x$-axis. Let us define 
this average:
 \begin{equation}\label{DelayYAv}
\tau_x(x_0,\theta) = \frac{1}{L_y}\int_{-L_y/2}^{L_y/2} \tau (x_0,y_0,\vect{v}_0)\dd y_0.
\end{equation}
By writing $\vect{v}_0 = (v_0^{(x)},v_0^{(y)})$, integrating
(\ref{dimensionalTauProblem}) and using (\ref{tauboundaryconditions}),
we obtain the following equation for $\tau_x$
\begin{equation}\label{unEval}
\begin{alignedat}{1}
&v_0^{(x)}\frac{\p \tau_x}{\p x_0} 
- \frac{|v_0^{(y)}| \, K(\vect{v}_0',\vect{v}_0)}{L_y}
- \lambda \tau_x + \lambda \int_V{T(\vect{u}_0,\vect{v}_0)\tau_x ( x_0, \vect{u}_0) \dd \vect{u}_0} \\ 
&\hspace{3cm}= -\left( 1 + \lambda \int_V 
T(\vect{u}_0,\vect{v}_0) K(\vect{u}_0,\vect{v}_0) \dd \vect{u}_0\right)\,. 
\end{alignedat}
\end{equation}
In the case where $T$ is the unbiased, fixed-speed, 2-dimensional turning kernel
given by \eqref{eq:turningKernel} and using (\ref{eq:DelayKernel}), 
we have $\vect{v}_0 = (v_0^{(x)},v_0^{(y)}) 
= s (\cos\theta,\sin\theta)$ and we can evaluate the second integral 
term in equation (\ref{unEval}) explicitly to be
\begin{equation*}
\int_V T(\vect{u}_0,\vect{v}_0) \, K(\vect{u}_0,\vect{v}_0) \, \dd \vect{u}_0 
= \int_{\theta - \pi}^{\theta + \pi} 
\frac{1}{2\pi} \frac{1}{\omega}|\theta - \theta_*|\dd\theta_* =
 \frac{\pi}{2\omega}.
\end{equation*}
Then (\ref{unEval}) can be rewritten as follows
\begin{equation}\label{RevExitTimeProb}
s\cos(\theta)\frac{\p \tau_x}{\p x_0} - \lambda \tau_x 
+
\frac{\lambda}{2 \pi} \int_{-\pi}^{\pi} \tau_x ( x_0, \phi) \dd \phi 
= -
\left(1 + \frac{\pi \lambda}{2\omega} - 
\frac{2 \, s \, A(\theta)}{L_y \, \omega}
\right),
\end{equation}
where $A(\theta)$ is defined by
$$
A(\theta) =
\left\{ 
\begin{array}{ll}
- (\pi + \theta ) \sin(\theta), \qquad &\mbox{for} \quad \theta \in 
(-\pi, -\pi/2], \\
\theta \sin(\theta), \qquad &\mbox{for} \quad \theta \in [-\pi/2, \pi/2], \\
(\pi - \theta ) \sin(\theta), \qquad\qquad\qquad &\mbox{for} \quad \theta \in 
[\pi/2,\pi]. \\
\end{array}
\right.
$$
\changed{Interestingly, the contribution of free turning on the right-hand
side of (\ref{RevExitTimeProb}) is given as $\pi\lambda/(2\omega)$, which can 
be explained using a simple averaging argument, because 
every tumble takes an average time of $\pi/(2/omega)$.}

The boundary conditions (\ref{tauboundaryconditions}) simplify to
\begin{equation} 
\label{tauboundaryconditionssim}
\begin{alignedat}{2}
\tau_x(L_x,\theta) &= 0\,,&\hspace{0.5cm}& \theta \in (-\pi/2, \pi/2)\,, \\ 
\tau_x(0,\theta) &= \tau_x(0,\pi -\theta) + 
\displaystyle \frac{\pi - 2 |\theta|}{\omega}\,,&&\theta \in (-\pi/2, \pi/2)\,. 
\end{alignedat}
\end{equation}
The numerical solution of (\ref{RevExitTimeProb})--(\ref{tauboundaryconditionssim}) can 
be further simplified by considering the symmetry in angle
$\tau_x(x_0, \theta ) = \tau_x(x_0,-\theta)$, i.e. it is sufficient to 
solve (\ref{RevExitTimeProb}) where $(x_0, \theta)$ are restricted to 
the domain $(0,L_x) \times (0,\pi)$ with boundary conditions 
(\ref{tauboundaryconditionssim}).

\subsubsection{Comparison between theory and experiments: loss of mass over 
time}
\changed{
In this section, we show that the transport theory with delays better explains 
the
experimental data with robots by considering the loss of mass over time, as we
did in Section \ref{subsec:comparison1}. 
In Figure~\ref{fig:comparison2}\subref{subfig:MassOverTime2}, we plot the mass remaining in the 
system 
against time. The solid (black) line represents the experimental data, whilst the 
results of the classical theory are shown as dotted (red) line. The dashed (blue) line shows a 
numerical 
solution of system \eqref{eq:restingstate:p}--\eqref{eq:restingstate:r}
that incorporates the finite reorientation time into the analysis. The 
numerical solution 
was achieved using a first order finite-volume method paired with an upwind 
scheme 
for \eqref{eq:restingstate:r}. For \eqref{eq:restingstate:p} we used 
$\Delta x = 1.183\,\text{m}/200$, $\Delta t = 10^{-3} \,\text{sec}$ and 
$\Delta \theta = \pi/20$. 
For \eqref{eq:restingstate:r} we used the same 
$\Delta t = 10^{-3} \,\text{sec}$ 
and 
a discretisation of $\Delta \eta = 3.38 \times 10^{-2}\,\text{sec}$ 
corresponding to the time 
it takes to turn from one velocity direction to the next.
Figure~\ref{fig:comparison2}\subref{subfig:MassOverTime2} demonstrates that 
the inclusion 
of turning delays provides an improved match to the experimental data.
}
\subsubsection{Comparison between theory and experiments: mean exit time problem}
\changed{
The mean exit time problem from Section \ref{subsec:comparison1b} can also be 
better modelled by the transport equation theory with suitable incorporated
delays as is demonstrated in Figure~\ref{fig:comparison2}\subref{subfig:MeanExitTime2}. 
The solid (black) line represents again the 
experimental data, whilst the classical results are shown as dotted (red) lines. 
The numerical solution of \eqref{RevExitTimeProb}
with the boundary conditions \eqref{tauboundaryconditionssim} is 
shown as the dashed (blue) line. This numerical solution was 
obtained using the same method as in Section~\ref{subsec:comparison1b} 
and we again plot the average over the initial 
pen as a bold dashed line. \changed{The bold dashed line indicates a predicted mean exit 
time of $152.43\,\text{sec}$  compared to the experimental value $156.74\,\text{sec}$, 
an error of approximately $2.7\%$ or $4.31\,\text{sec}$. This represents a strong 
improvement to the discrepancy of $12.3\%$ seen for the model that neglected the 
turning events (dotted red line) and goes to show that turning times are indeed 
non-negligible and can be built into our model in a consistent manner.}
}
\section{\changed{Incorporation of a signal gradient}}
\label{sec:signal}
In this section, we are aiming to formulate velocity jump models that incorporate changing turning frequencies $\lambda$.
In particular, we are interested in turning frequencies that depend on the current velocity of the robot as well
as its position in the domain, i.e. $\lambda = \lambda(\vect{x}, \vect{v})$. The general velocity jump model for this case
can be formulated as (cf. \eqref{eq:classicalvjp})
\begin{equation}\label{eq:gradient:VJP}
	\pd{p}{t} + \vect{v}\cdot\nablax p = -\lambda(\vect{x}, \vect{v})\,p + \int_V \lambda(\vect{x}, \vect{u})\,
	T(\vect{v}, \vect{u})\,p(t, \vect{x}, \vect{u}) \DD{\vect{u}}\,,
\end{equation}
with the boundary conditions given in \eqref{eq:classicalBC}. Similarly, we can formulate this system by
incorporating the resting period (cf. \eqref{eq:restingstate:p}--\eqref{eq:restingstate:r})
\begin{equation}\label{eq:gradient:restingstate}
\begin{alignedat}{3}
	&\pd{p}{t} &&+ \vect{v}\cdot \nablax p && = -\lambda(\vect{x}, \vect{v})\,p(t, \vect{x}, \vect{v}) + r(t, \vect{x}, \vect{v}, 0^+)\,,\\
	&\pd{r}{t} &&- \pd{r}{\eta} && = \int_V \lambda(\vect{x}, \vect{u})\,p(t, \vect{x}, \vect{u})\, T(\vect{v}, \vect{u*})\, \delta(\eta - K(\vect{v}, \vect{u}))\DD{\vect{u}}\,,
\end{alignedat}
\end{equation}
with boundary conditions \eqref{eq:robotics:restingstate:BC}. The system \eqref{eq:gradient:restingstate} can again
be formulated in the form of a delay differential equation (cf. \eqref{eq:delayDE})
 \begin{equation}\label{eq:robotics:signal:delayDE}
	\frac{\partial p}{\partial t} +\vect{v}\cdot\nablax p 
	= -\lambda(\vect{x},\vect{v})\, p 
	+ \int_V\lambda(\vect{x}, \vect{u})\,T(\vect{v},\vect{u})\,p(t - K(\vect{v},\vect{u}), \vect{x}, \vect{u})\DD{\vect{u}}\,,
\end{equation}
where boundary conditions take the form \eqref{eq:delayBC}. Similarly to the derivation in Section~\ref{subsec:turning:mfpt},
one can derive the backwards problem, with the mean first passage time equation taking the form (cf. \eqref{eq:meanExitTime})
\begin{equation}
	\label{eq:robotics:signal:MFPT}
	\begin{alignedat}{1} 
	& \vect{v_0}\cdot\nabla_{\vect{x_0}} \tau - \lambda(\vect{x_0}, \vect{v_0}) \tau + 
	\lambda(\vect{x_0}, \vect{v_0}) \int_V T(\vect{u},\vect{v_0})\,\tau ( \vect{x_0}, \vect{u}) \DD{\vect{u}} \\ 
	& \hspace{2cm} = - \left( 1 +  \lambda(\vect{x_0}, \vect{v_0}) \int_V T(\vect{u},\vect{v_0})\, K(\vect{u},\vect{v_0}) \DD{\vect{u}} \right)\,,
	\end{alignedat}
\end{equation}
with boundary conditions given in \eqref{eq:meanExitTimeBC}.

\subsection{Experiments with a signal gradient}

In order to compare these generalised velocity jump models to experimental results, 
we introduce an external signal 
into the robot experiments presented in Section~\ref{subsec:epuck}. The signal 
is incorporated in the form of a colour gradient
that can be measured by the light sensors on the bottom of the \epuck robots. 
The colour gradient is layed out in such
a way that it changes along the $x$-axis in Figure~\ref{fig:arenaSketch} with the 
darker end closer to the target area.
The reaction of the robots to this colour gradient is implemented using the internal
variable $z$ and a changing
turning frequency $\lambda(z)$ that are updated according to
\begin{equation}
\begin{alignedat}{2}
\frac{\D z}{\D t} & = \frac{S - z}{t_a}\,,\\
\lambda & = \lambda_0 + \lambda_0\left(1 - \alpha (S - z)\right)\,,
\end{alignedat}
\end{equation}
where $S\in[0, 1]$ represents the measured signal with increasing values of $S$ indicating 
a darker colour in the gradient.
The way the turning frequency is changed is motivated by models of bacterial chemotaxis 
\cite{Erban:2004:ICB}.

According to results from \cite{Erban:2005:STS}, a macroscopic density formulation 
for the robotic system is given
through the hyperbolic chemotaxis equation
\begin{equation}\label{eq:robotics:gradientChemo}
\frac{1}{\lambda_0}\pd{^2n}{t^2} + \pd{n}{t} 
= \frac{s^2}{d\lambda_0}\Laplace n - 
\nabla\cdot\left(n \frac{\alpha\lambda_0 s^2t_a}{d\lambda_0(1+\lambda_0t_a)}\nabla S\right)\,,
\end{equation}
where $S:\Omega\mapsto\R$ indicates the colour gradient and $n(t,\vect{x})$ 
describes the concentration of robots in $\Omega$.
Equation \eqref{eq:robotics:gradientChemo} can be approximated by the velocity 
jump process \eqref{eq:gradient:VJP}
with the form for the turning frequency given by
\begin{equation}\label{eq:robotics:signal:lambdaApp}
\lambda(\vect{x}, \vect{v}) = 
\lambda_0 - \gamma\,\vect{v}\cdot\nabla S(\vect{x})\,,\qquad \gamma = 
\frac{\alpha t_a\lambda_0}{1 + \lambda_0 t_a}\,.
\end{equation}
Because the gradient of the colour signal $S$ was chosen to be parallel to the $x$-axis 
in the experimental
setting, we can again simplify the formulation of the exit time problem 
\eqref{eq:robotics:signal:MFPT} by averaging along the $y$-axis. The resulting equation 
takes the form

\begin{equation}\label{eq:robotics:signal:MFPT:simple}
s\cos\theta\,\pd{\bar{\tau}}{x} + 
\frac{2s\sin\theta}{L_y\omega} \min(\theta, \pi - \theta) - \lambda(x,\theta)\,\bar{\tau}
+ \frac{\lambda(x,\theta)}{\pi}\int_0^\pi\!\!\!\bar{\tau}(x, \theta_*)\DD{\theta_*} 
= -1 - \lambda(x, \theta)\frac{\pi}{2\omega}\,,
\end{equation}
where $\lambda(x, \theta)$ is given through
\begin{equation}\label{eq:robotics:signal:lambdaApp1D}
\lambda(x, \theta) = \lambda_0 - \gamma\, s\cos\theta\,\pd{S(x)}{x}\,.
\end{equation}
Because the colour changes linearly along the $x$-axis, we approximate the 
signal $S(x)$ by a linear function. The values
at the end-points were taken directly from robot measurements and hence, $S(x)$ takes
the form
\begin{equation}
S(x) = 0.23 + 0.39\frac{x}{L_x}\,,\qquad \pd{S}{x} \approx 0.33\unit{m}^{-1}\,.
\label{linapp}
\end{equation}
We will use this linear form of $S(x)$ for all comparisons between experimental 
data and the derived models.

\subsection{Comparison between models and experimental results}
\label{subsec:robotics:gradient:comparison}

\begin{figure}[ht]
\centering
\subfigure[]{
\includegraphics[width=0.45\textwidth, height=0.35\textwidth]{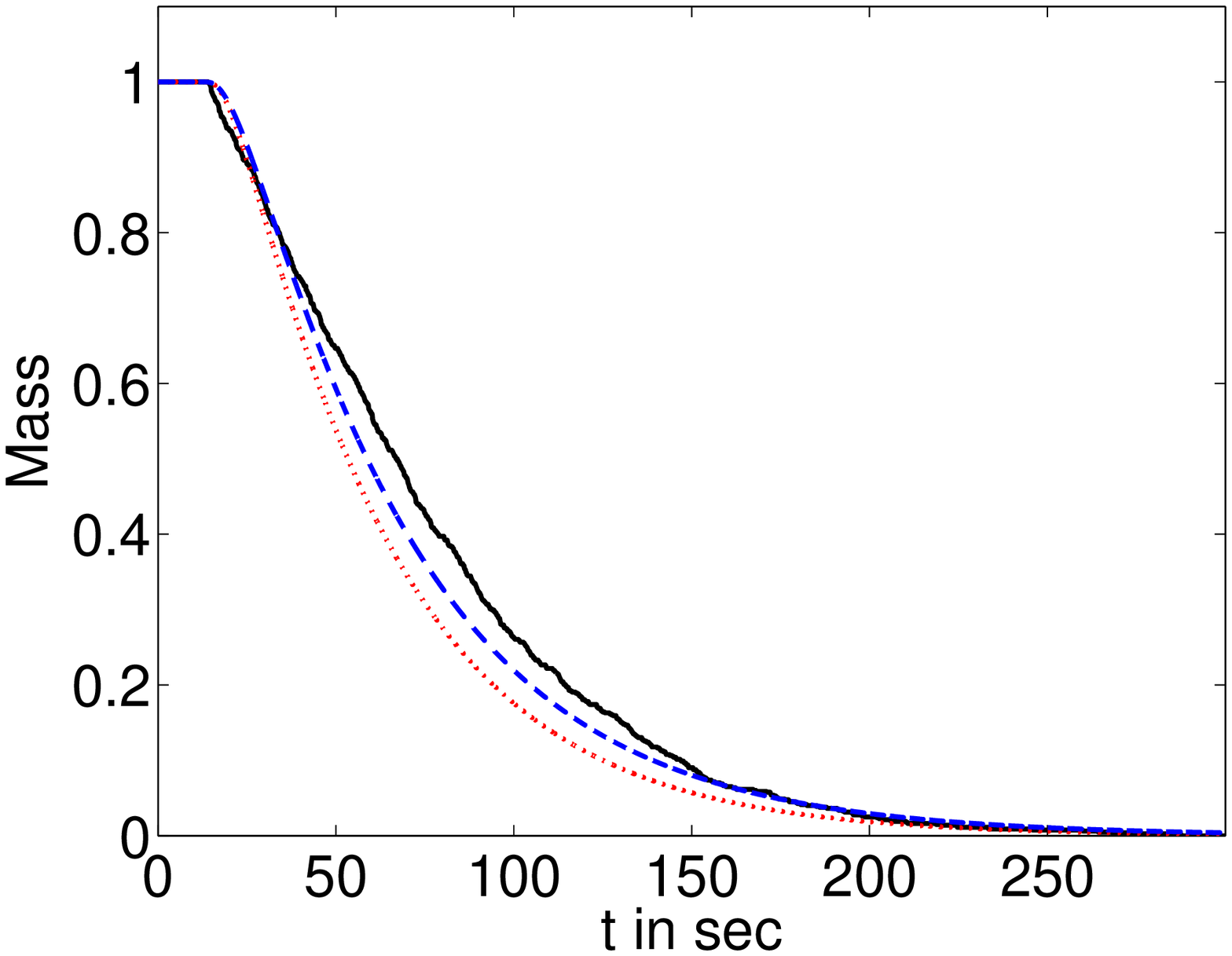}
}
\hspace{0.03\textwidth}
\subfigure[]{
\includegraphics[width=0.45\textwidth, height=0.35\textwidth]{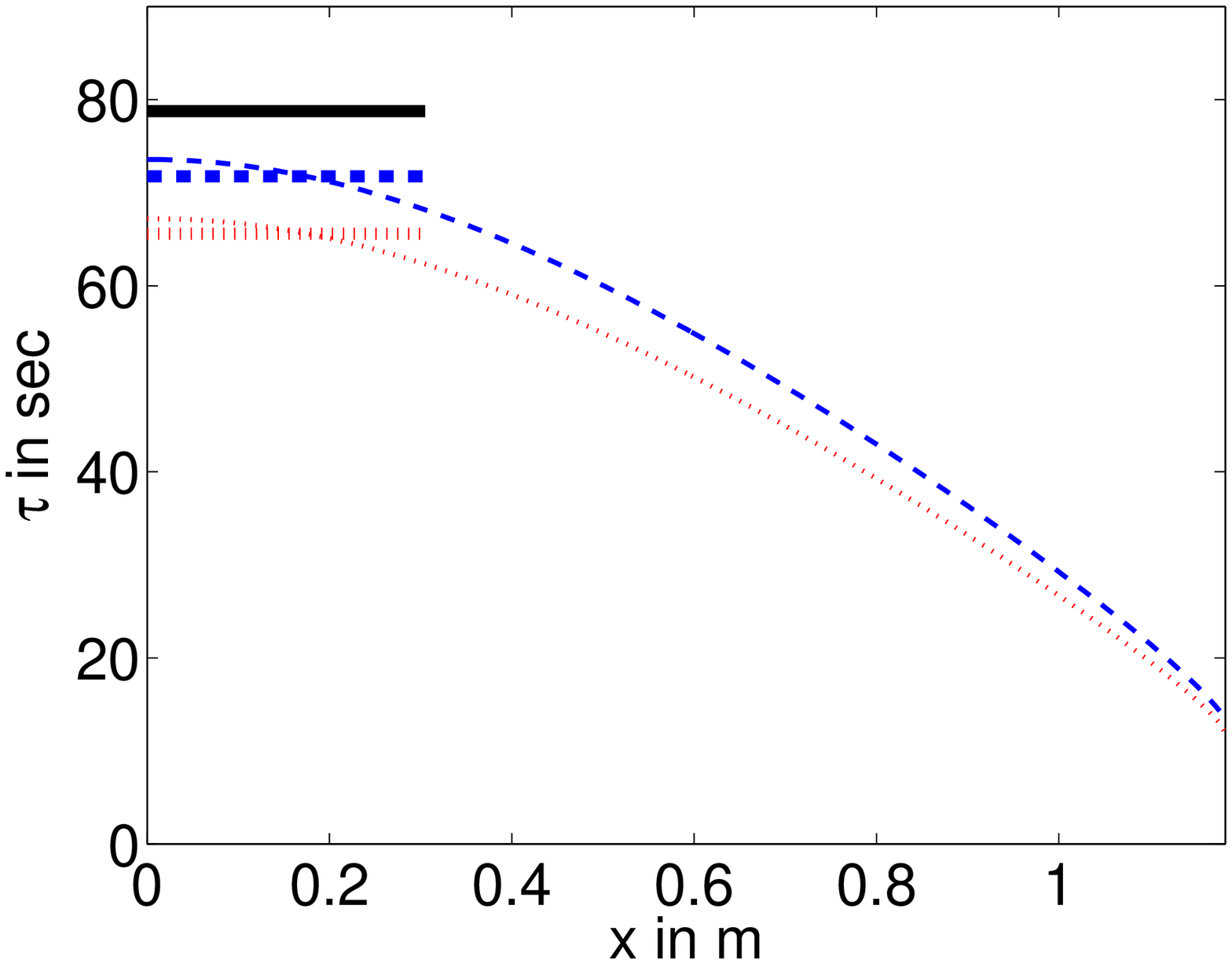}
}
\caption{{\it (Colour available online) 
Comparison between velocity jump process and experimental data for experiment 
including colour gradient.}
(a) {\it Mean mass in system over time. Solid line (black): experimental data;
dotted line (red): numerical solution of $\eqref{eq:gradient:VJP}$;
dashed line (blue): numerical solution of $\eqref{eq:gradient:restingstate}$. 
Turning frequency $\lambda(\vect{x}, \vect{v})$ as given
in $\eqref{eq:robotics:signal:lambdaApp}$.}\newline
(b) {\it Mean exit time averaged over all velocities. 
Solid line (black): experimental data;
dotted line (red): numerical solution of $\eqref{eq:robotics:signal:MFPT:simple}$ 
for $\omega = \infty$; 
bold dotted line (red): average of dotted line over $\Omega_0$;
dashed line (blue): numerical solution of $\eqref{eq:robotics:signal:MFPT:simple}$ 
for $\omega = 4.65\unit{rad}\unit{sec}^{-1}$;
bold dashed line (blue): average of dashed line over $\Omega_0$. Turning frequency 
$\lambda(x, \theta)$ as given in $\eqref{eq:robotics:signal:lambdaApp1D}$}\newline
{\it For both plots parameters and numerical methods are given in the text.}}
\label{fig:robotics:Gradient}
\end{figure}
We now want to compare the experimental data to the generalised 
velocity jump models presented in
\eqref{eq:gradient:VJP}--\eqref{eq:robotics:signal:lambdaApp1D}. 
The numerical solutions were achieved using 
the exact same methods and parameters as in Section~\ref{subsec:comparison1} and 
the results can be seen in Figure~\ref{fig:robotics:Gradient}. 
The parameter values used for the robots are
$\lambda_0 = 0.25\unit{sec}^{-1}$, $\alpha = 8$, 
$t_a = 10\unit{sec}$ and $s = 5.8 \times 10^{-2} \,\text{m}/\text{sec}$.
The experimental procedure was equivalent to the one 
presented in Section~\ref{sec:experiments}, i.e.
we repeated the experiment 50 times with 16 robots, 
each time waiting until all of the 16 robots
have left the arena.

In Figure~\ref{fig:robotics:Gradient}(a) we plot the mass 
left in the system over time. The 
solid (black) line represents the percentage of robots still 
in the arena at that point in time.
The dotted (red) line is a numerical solution of the velocity 
jump equation \eqref{eq:gradient:VJP} with 
the corresponding boundary conditions \eqref{eq:classicalBC}. 
The dashed (blue) line is a numerical solution 
of the velocity jump system with resting state given 
in \eqref{eq:gradient:restingstate} and boundary conditions 
as in \eqref{eq:robotics:restingstate:BC}.

In Figure~\ref{fig:robotics:Gradient}(b) we plot the 
mean exit time in dependence of position
along the $x$-axis. The horizontal solid (black) line again 
indicates the experimentally measured exit time of $78.77\unit{sec}$.
The dotted (red) line shows a numerical solution of 
\eqref{eq:robotics:signal:MFPT:simple} with instant turning, 
i.e. $\omega = \infty$.
The dashed (blue) line shows a numerical solution of 
\eqref{eq:robotics:signal:MFPT:simple} with 
$\omega = 4.65\unit{rad}\unit{sec}^{-1}$. For both of 
these solutions the boundary conditions are given 
in \eqref{eq:meanExitTimeBC}.
The bold horizontal lines again indicate the average over 
the initial pen $\Omega_0$.

In both plots in Figure~\ref{fig:robotics:Gradient}, we 
see that
the models including finite turning delays (represented 
through the dashed (blue) lines) give an improved match compared to the
models without this delay. 
The numerically estimated exit time for the model with instant 
turning ($\omega = \infty$) is $65.59\unit{sec}$ (error of
$16.7\%$ compared to experimental data); with finite turning 
times it is $71.76\unit{sec}$ (error of $8.9\%$).
The remaining difference between the models and the experimental 
data can be explained by noisy measurement
of the signal $S(\vect{x})$ as well as the fact that we used 
linear approximation (\ref{linapp}) averaged over all robots 
to obtain the numerical results. 
We can conclude from this brief study of robot experiments 
including a colour gradient signal that this signal indeed 
improves the target finding capacity of the robots and that 
the models developed in 
Section~\ref{sec:delay} can be generalised to incorporate 
turning frequencies that change according
to external signals.

\section{Discussion}
\label{secdiscussion}

In this paper, we have studied an implementation of a run-and-tumble searching
strategy in a robotic system. The algorithm implemented by the robots is 
motivated by a biological system -- behaviour of the flagellated bacterium
{\it E. coli}. Bio-inspired algorithms are relatively common in swarm
robotics. Algorithms based on behaviour of social insects have been
 implemented previously in the literature, see for example
\citet{GJJGACT,Krieger:2000:ATA,Webb:2000:WDR} and \citet{Fong:2003:SSI}.
One of the challenges of bio-inspired algorithms is that robots do not have
the same sensors as animals. For example, {\it E. coli} bias
their movement according to extracellular chemical systems. In biological
models, chemical signals often evolve according to the solution
of reaction-diffusion partial differential equations 
\citep{Franz:2012:HMI,Franz:2013:TWH}. Therefore, an implementation
of the full run-and-tumble chemotactic model in the robotic system requires 
either special sensors for detecting chemical signals, e.g. robots for 
odour detecting \citep{Russell:2001:SRA}, or replacing chemical signals 
by suitable caricatures of them, e.g. using glowing floor for 
\epuck robots \citep{Mayet:2010:AFV}. 

The main goal of this paper is to study how the mathematical 
theory developed for {\it E. coli} applies to the robotic system based 
on \epuckper{}s. Thus we do not focus on technological challenges 
connected with sensing changing chemical signals or their analogues 
\citep{Russell:2001:SRA,Mayet:2010:AFV}, \changed{ we do, however, incorporate
a constant signal in order to show that the developed theory works
for unbiased as well as biased velocity jump processes.} 
If the collisions between particles
(robots or bacteria) and reorientation times can be neglected, then this velocity jump process
is described by the transport equation (\ref{eq:classicalvjp}) 
\changed{or \eqref{eq:gradient:VJP} (in the biased case)}
and the long time behaviour is given by a drift-diffusion 
equation \citep{Hillen:2000:DLT}. In Section \ref{subsec:collisions}
we show that collisions between robots are negligible in
our experimental set up. However, we still observe quantitative
differences between the results based on the transport equation 
(\ref{eq:classicalvjp}) and robotic experiments.

In Section \ref{sec:delay} we identify turning delays as the
main mechanism contributing to differences between the 
mathematical theory developed for {\it E. coli} and the results
of experiments with \epuckper{}s. We introduce the resting state 
in equations \eqref{eq:restingstate:p}--\eqref{eq:restingstate:r}
and then derive the transport equation with delay
(\ref{eq:delayDE}). Our delay term is different from models
of tumbling of {\it E. coli}, because the underlying physical
process is different. Tumbling times of {\it E. coli}
are exponentially distributed, i.e. they can be explicitly
included in mathematical models by using transport equations which 
take into account probabilistic changes to and from the resting 
(tumbling) state \citep{Erban:2004:ICB}. In the case of robots, 
the turning time depends linearly on the turning 
angle. The selection of new direction is effectively instant and
the main contributing factor to turning delays is the finite 
turning speed of robots. \changed{In Section~\ref{sec:signal}
we apply the developed theory to an experiment incorporating
an external signal and show that similar transport equations
can be developed for this situation.}

We have studied a relatively simple searching algorithm 
motivated by \emph{E. Coli} behaviour, but the transport
equations and velocity jump processes naturally appear in 
modelling of other biological systems, such as modelling 
chemotaxis of amoeboid cells \citep{Erban:2007:TEA} or swarming 
behaviour as seen in various fish, birds and insects 
\citep{COP,Erban:2012:ICB}. We conclude that the same delay terms 
as in (\ref{eq:delayDE}) would be applicable whenever we implement
these models in \epuckper{}s. From a mathematical point of view, it
is also interesting to consider coupling of (\ref{eq:delayDE}) with 
\changed{changing extracellular signals}, because signal transduction also 
has its own delay which can be modelled using velocity jump models 
with internal dynamics \citep{Franz:2013:TWH,Erban:2004:ICB,Xue:2009:MMT}. 
\changed{Considering higher densities of robots, the transport equation
formalism needs to be further adapted to incorporate the effects
of robot-robot interactions. We have recently investigated this
problem and reported our results in \cite{Franz_2014}.} 

\section*{Acknowledgements}

The research leading to these results has received funding from
the European Research Council under the {\it European Community's}
Seventh Framework Programme {\it (FP7/2007-2013)} /
ERC {\it grant agreement} No. 239870; and from the Royal Society
through a Research Grant. Christian Yates would like to thank 
Christ Church, Oxford for support via a Junior Research Fellowship.
Radek Erban would also like to thank the Royal Society for
a University Research Fellowship; Brasenose College, University of Oxford,
for a Nicholas Kurti Junior Fellowship; and the Leverhulme Trust for
a Philip Leverhulme Prize.

\appendix
\renewcommand\thesection{Appendix \Alph{section}}

\section{Robot specifications}
\label{app:epuck}

A photo of a collection of \epuck robots and the arena are given 
in Figure \ref{CollPhotoArena}. Full details of the \epuck specifications 
are:

\vskip -6mm \rule{0pt}{0pt}

\begin{itemize}
\itemsep -1.3mm 
\item[i.]
Diameter: 75 mm.
Height: 50 mm.
Weight: 200g.
\item[ii.]
Speed throughout experiments: $5.8 \times 10^{-2} \,\text{m}\,\text{sec}^{-1}$, 
(max speed: $0.13 \,\text{m}\,\text{sec}^{-1}$).
\item[iii.]
Turning speed throughout experiments: $4.65 \,\text{rad}\,\text{sec}^{-1}$.
\item[iv.]
Processor: dsPIC 30 CPU @ 30 MHz (15 MIPS), (PIC Microcontroller.)
\item[v.]
RAM: 8 KB. Memory: 144 KB Flash.
\item[vi.]
Autonomy: 2 hours moving. 2 step motors. 3D accelerometers.
\item[vii.]
8 infrared proximity and light, (TCRT1000)
\item[viii.]
Colour camera, 640x480,
\item[ix.]
8 LEDs on outer ring, one body LED and one front LED,
\item[x.]
3 microphones, forming a triangle allowing the determination of the direction of audio cues.
\item[xi.]
1 loudspeaker.
\end{itemize}

\vskip -8mm \rule{0pt}{0pt}

\begin{figure}[ht]
\centerline{
\hskip 2mm
\raise 5.5cm \hbox{\raise 0.9mm \hbox{(a)}}
\hskip -5mm
\includegraphics[height=5.5cm]{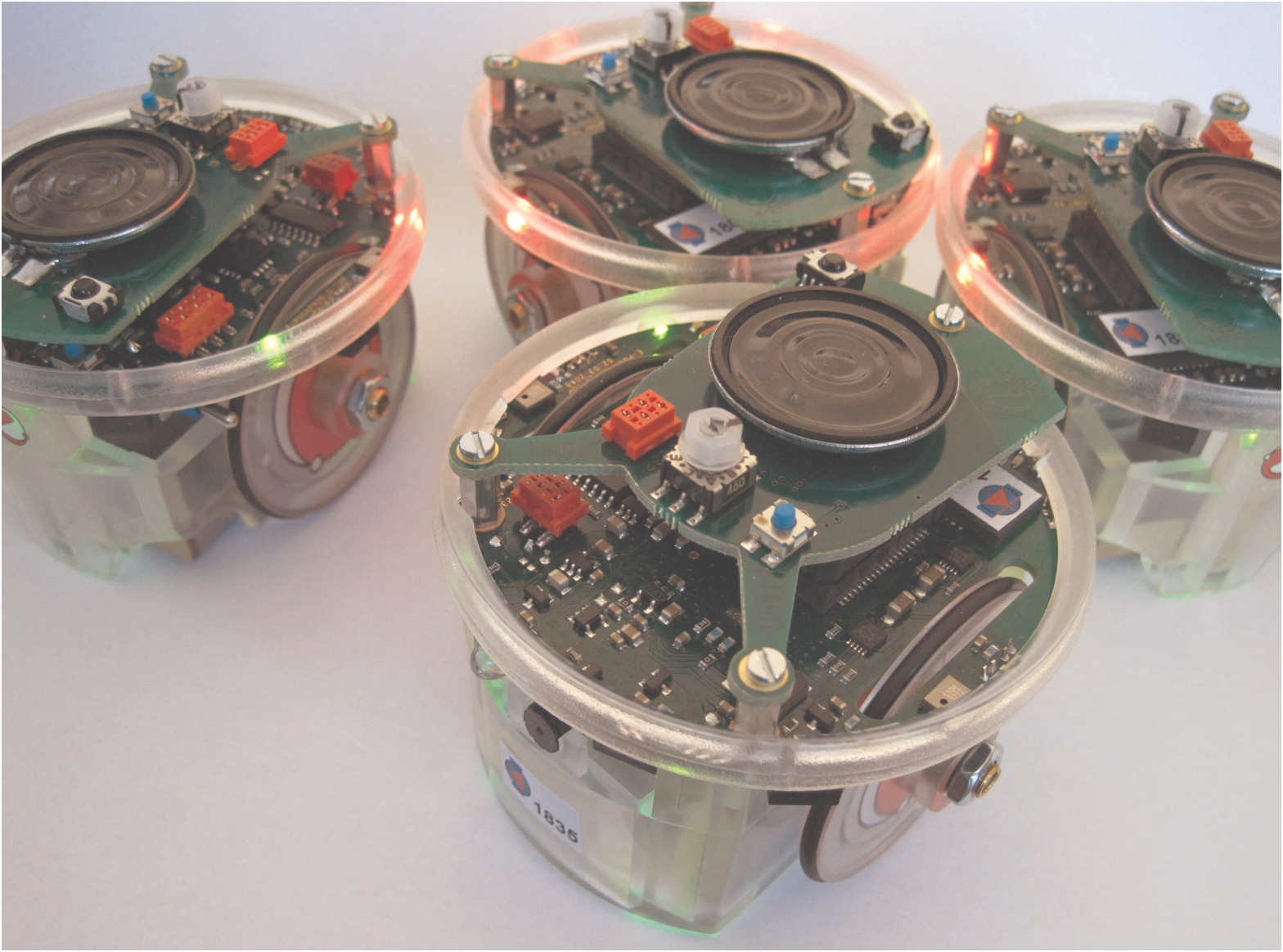}
\hskip 12mm
\raise 5.5cm \hbox{\raise 0.9mm \hbox{(b)}}
\hskip -5mm
\includegraphics[height=5.5cm]{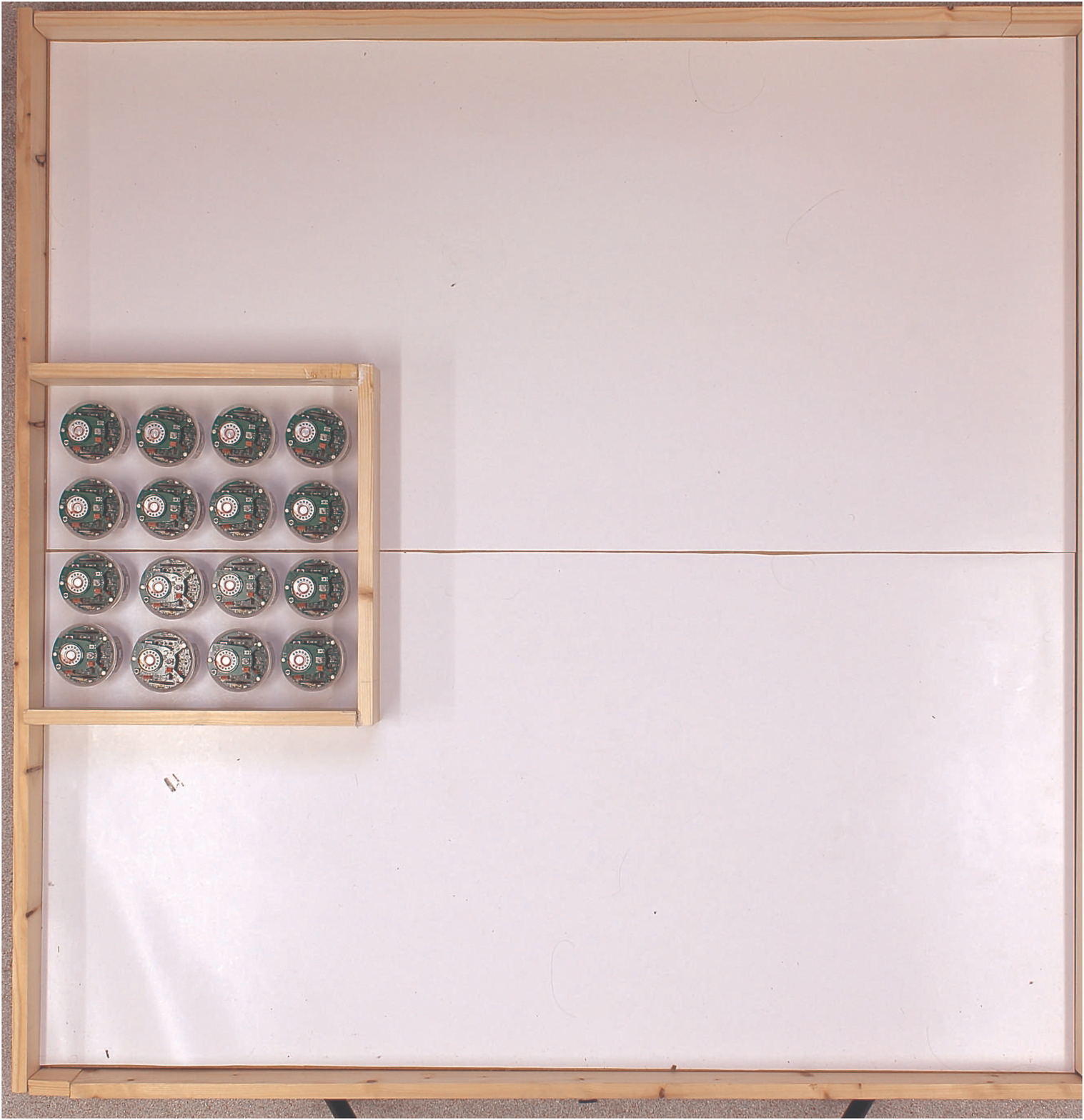}
}
\caption{(a) {\it Collection of \epuck robots.} 
(b) {\it Photo showing the arena built for the experiments, 
removable pen and edges made from cut down medium density fibre.}}
\label{CollPhotoArena}
\end{figure}

\section{Derivation of adjoint boundary condition (\ref{backwardbcs})}
\label{BCDERIVATION}

Using (\ref{boundarypartition}) and (\ref{velocitypartition}),
the last term in (\ref{pomeqMstar}) can be rewritten as follows
\begin{alignat*}{1}
&\int_{-\infty}^{\infty} 
\int_V
\int_{\p\Omega} 
p(t,\vect{x},\vect{v}) q(t,\vect{x},\vect{v}) 
[\vect{v} \cdot \vect{n} ]\, \dd S_x \, \dd\vect{v} \, \dd t
\\
&\hspace{1cm}
=
\int_{-\infty}^{\infty} 
\int_{\bdref \cup \bdtar} 
\int_{V^+} 
\big\{p(t,\vect{x},\vect{v}) q(t,\vect{x},\vect{v}) 
- p(t,\vect{x},\vect{v}') q(t,\vect{x},\vect{v}')  \big\} 
[\vect{v} \cdot \vect{n}] \, \dd\vect{v} \, \dd S_x  \, \dd t\,,
\end{alignat*}
where $\vect{v}'$ is given by (\ref{eq:reflection}).
Separating the above integral into the cases of 
$\bdref$ and $\bdtar$, and using the boundary condition
\eqref{eq:delayBC}, we have
\begin{alignat*}{1}
&
\int_{-\infty}^{\infty} 
\int_V
\int_{\p\Omega} 
p(t,\vect{x},\vect{v}) q(t,\vect{x},\vect{v}) 
[\vect{v} \cdot \vect{n} ]\, \dd S_x \, \dd\vect{v} \, \dd t
\\
&\hspace{1cm} = \int_{-\infty}^{\infty} 
  \int_{\bdref} 
  \int_{V^+} 
  \left\{
  p(t,\vect{x},\vect{v}) q(t,\vect{x},\vect{v}) 
  - p\left(t- K(\vect{v}',\vect{v}),\vect{x},\vect{v}\right) 
  q(t,\vect{x},\vect{v}') \right\} 
  [\vect{v} \cdot \vect{n}] \, \dd\vect{v} \, \dd S_x \, \dd t
  \\
&\hspace{1.4cm}  
-
\int_{-\infty}^{\infty} 
\int_{\bdtar} 
\int_{V^+} 
\left\{  
p\left(t,\vect{x},\vect{v}'\right) 
q\left(t,\vect{x},\vect{v}'\right)  
\right\} [\vect{v} \cdot \vect{n}] \, \dd\vect{v} \, \dd S_x  \, \dd t\,.  
\end{alignat*}
We shift the time variable in the first term on the right hand side to
deduce
\begin{alignat*}{1}
&
\int_{-\infty}^{\infty} 
\int_V
\int_{\p\Omega} 
p(t,\vect{x},\vect{v}) q(t,\vect{x},\vect{v}) 
[\vect{v} \cdot \vect{n} ]\, \dd S_x \, \dd\vect{v} \, \dd t
\\
&\hspace{1cm} = \int_{-\infty}^{\infty} 
  \int_{\bdref} 
  \int_{V^+} 
  p(t,\vect{x},\vect{v})
  \left\{  
  q(t,\vect{x},\vect{v}) - 
  q\left(t+ K(\vect{v}',\vect{v}),\vect{x},\vect{v}' \right)  
  \right\} [\vect{v} \cdot \vect{n}] \, \dd\vect{v} \, \dd S_x  \, \dd t \\
&\hspace{1.4cm}
-
\int_{-\infty}^{\infty} 
\int_{\bdtar} 
\int_{V^+} 
\left\{  
p\left(t,\vect{x},\vect{v}'\right) 
q\left(t,\vect{x},\vect{v}'\right)  
\right\} [\vect{v} \cdot \vect{n}] \, \dd\vect{v} \, \dd S_x  \, \dd t 
\,.
\end{alignat*}
The first term on the right hand side is zero
because $q(t+ K(\vect{v}',\vect{v}),\vect{x},\vect{v}') 
= q(t,\vect{x},\vect{v})$ in (\ref{backwardbcs}). 
The second term vanishes when $q(t,\vect{x},\vect{v}') = 0$.
Thus we conclude that the last term in (\ref{pomeqMstar}) 
is equal to zero when $q$ satisfies the boundary conditions 
(\ref{backwardbcs}).

\end{document}